\documentclass{article}

\usepackage{bm}
\usepackage{graphicx}
\usepackage{subcaption}
\usepackage{booktabs}
\usepackage{multirow}
\usepackage{array}
\usepackage{xcolor}
\usepackage{colortbl}
\usepackage{enumitem}
\usepackage{microtype}

\definecolor{bestcol}{HTML}{D6F5D6}
\definecolor{hblue}{HTML}{1f77b4}
\definecolor{horange}{HTML}{ff7f0e}
\definecolor{hgreen}{HTML}{2ca02c}
\definecolor{hred}{HTML}{d62728}
\definecolor{hpurple}{HTML}{9467bd}
\definecolor{hgray}{HTML}{7f7f7f}
\definecolor{hbrown}{HTML}{8c564b}
\definecolor{hcyan}{HTML}{17becf}

\newcommand{\herald}{\textsc{Herald}}
\newcommand{\bh}{\mathbf{h}}
\newcommand{\bv}{\mathbf{v}}

\newcommand{\bp}{\mathbf{p}}
\newcommand{\bmu}{\boldsymbol{\mu}}
\newcommand{\SW}{\mathbf{S}_W}
\newcommand{\SB}{\mathbf{S}_B}

\usepackage{hyperref}

\usepackage[accepted]{icml2026}

\usepackage{amsmath}
\usepackage{amssymb}
\usepackage{mathtools}
\usepackage{amsthm}

\usepackage[capitalize,noabbrev]{cleveref}

\theoremstyle{plain}
\newtheorem{theorem}{Theorem}[section]
\newtheorem{proposition}[theorem]{Proposition}

\theoremstyle{definition}

\theoremstyle{remark}

\usepackage[textsize=tiny]{todonotes}

\icmltitlerunning{Harmfulness Propagation Dynamics in Large Language Models}

\begin{document}

\twocolumn[
  \icmltitle{Harmfulness Propagation Dynamics: Layer-wise Trajectories of Adversarial Intent in Large Language Models}

  \icmlsetsymbol{equal}{*}

\begin{icmlauthorlist}
    \icmlauthor{Noor Islam S. Mohammad}{equal,yyy}
    \icmlauthor{Uluğ Bayazıt}{comp}
\end{icmlauthorlist}

\icmlaffiliation{yyy}{Department of Computer Science, İTÜ, İstanbul, Türkiye}
\icmlaffiliation{comp}{Faculty of Computer Engineering, İTÜ, İstanbul, Türkiye}

\icmlcorrespondingauthor{Noor Islam S. Mohammad}{islam23@itu.edu.tr}

  \icmlkeywords{LLM safety, input moderation, mechanistic interpretability, representation geometry, jailbreak detection}

  \vskip 0.3in
]

\printAffiliationsAndNotice{}

\begin{abstract}
We identify \textbf{Harmfulness Propagation Dynamics (HPD)}: for harmful prompts, the projection of the last-token hidden state onto a learned harm direction rises monotonically with transformer depth, whereas benign prompts remain flat or oscillatory. This cross-layer signature reflects harmful intent as a \emph{progressively resolved} semantic property: surface form appears early, while pragmatic intent consolidates later, making the \emph{trajectory shape} more informative than any single-layer snapshot. Moreover, LDA-based harm directions, learned per layer, remain stable across random splits (pairwise cosine similarity $>0.97$), supporting the projection sequence as a reproducible structured signal. Building on HPD, we introduce \textbf{\herald{}} (\textbf{H}armful \textbf{E}ncoding \textbf{R}ecognition via \textbf{A}ctivation \textbf{L}ayer \textbf{D}ynamics). This lightweight input moderator extracts a seven-dimensional feature record, slope, curvature, monotonicity, onset layer, and related statistics from the cross-layer projection sequence and classifies it with a 288-parameter MLP. \herald{} stores one $d$-dimensional direction per layer ($262$\,KB for a 32-layer, $d{=}4096$ model), requires no gradient computation during training, and adds only $2.6{\times}10^{-6}$ prefill FLOPs at inference. Across eight prompt-harmfulness benchmarks and four model families, \herald{} achieves an average F1 of $89.3$ on OLMo2-7B, surpassing all tested guard models on adversarial jailbreak detection ($98.4$ vs.\ $96.9$ F1) and outperforming prior latent-based methods by $2.3$-$4.1$ F1 points on every backbone. Per-instance trajectories provide machine-readable audit records that reveal \emph{when} and \emph{how} harmfulness emerges, offering an interpretability advantage over single-layer approaches. 
\end{abstract}

\section{Introduction}
\label{sec:intro}

Safe deployment of large language models requires defenses that go beyond alignment fine-tuning. Even well-aligned models remain vulnerable to adversarial and indirect prompts~\citep{perez2022redteaming,greshake2023indirect,ganguli2022redteaming,carlini2023aligned}, and alignment may come at the cost of
general capability~\citep{askell2021general}. Input moderation screening requests before generation is a complementary safeguard that blocks unsafe prompts while avoiding the full cost of a forward pass on harmful inputs. Existing moderators occupy two extremes. Guard models~\citep{markov2023openai,vidgen2023simple} are accurate but incur the expense of an additional large model (${\approx}14$\,GB for a 7B guard). Latent-based methods~\citep{thilo2024latentguard,li2025safetylayers} are lightweight but commit to a \emph{single} layer's hidden state, discarding information carried by the progression of representations across depth.

\textbf{Our starting point: harmfulness is a progressively resolved signal.}
Prior work on transformer representation geometry shows that different linguistic properties are encoded at different depths: syntax in early layers, semantics in middle layers, and task-relevant pragmatics in late layers~\citep{jawahar2019bert,tenney2019bert,geva2021transformer}. We ask whether \emph{harmfulness} follows this pattern and specifically whether the trajectory of a prompt's representation projected onto a harm direction, across all layers, is itself a discriminative signal.

\textbf{Core observation (HPD).} Projecting each layer's last-token hidden state onto a per-layer LDA harm direction yields a trajectory $\{p_l\}_{l=1}^L$ that
differs sharply between harmful and benign prompts (Section~\ref{sec:hpd}). For harmful inputs, particularly jailbreaks, the trajectory rises steadily and monotonically from near-zero in early layers to strongly positive values in late layers. Benign prompts remain flat or oscillatory, with no systematic directional
growth. 

We call this \textbf{Harmfulness Propagation Dynamics (HPD)}. Critically, HPD is \emph{not} merely a restatement of the well-known fact that
late-layer representations are more discriminative. The trajectory's \emph{shape} carries information beyond the terminal value: the onset layer, the monotonicity of growth, and the curvature each contribute independently (Section~\ref{sec:abl_features}), and the gap between trajectory and terminal-only classification is largest precisely for adversarial jailbreaks, the most practically important detection target.

\textbf{Relation to representation engineering.} \citet{Zou2023Privacy-Friendly} show that linear directions in activation space can steer and probe model behavior. HPD extends this insight in a distinct direction: rather than learning a single probe or steering vector, we track how a harm direction evolves
\emph{across layers} and treat the resulting trajectory as a first-class data object for classification. This cross-layer dynamics view is orthogonal to activation-space probing at a fixed depth.

\textbf{\herald{}} exploits HPD through four steps: (i) Learn a per-layer LDA harm direction $\bv_l$ using only a single gradient-free forward pass over the
training set; (ii) project each layer's last-token hidden state onto $\bv_l$ to obtain scalar $p_l$; (iii) extract a compact seven-dimensional feature vector
$\boldsymbol{\phi}(\bp)$ capturing the trajectory's slope, curvature, monotonicity, onset layer, and related statistics; and (iv) classify $\boldsymbol{\phi}$ with a 288-parameter MLP trained in seconds on CPU.

\textbf{Contributions:}
\begin{itemize}[topsep=2pt, itemsep=1pt, leftmargin=*]
  \item We identify and formally characterize \textbf{Harmfulness Propagation Dynamics}, a cross-layer signature of harmful prompts in which per-layer harm projections rise monotonically with depth. We prove that LDA harm directions converge to stable axes (pairwise cosine similarity $>0.97$ across five splits)
  and show that this stability is necessary for trajectory-based classification.
  \item We introduce \textbf{\herald{}}, a gradient-free trajectory moderator requiring $O(Ld)$ memory ($262$\,KB for a 32-layer model) and negligible runtime
  overhead ($2.6{\times}10^{-6}$ of prefill FLOPs), distinguishing it from guard models and full-covariance latent methods.
  \item \herald{} surpasses all tested guard models on adversarial jailbreak detection and outperforms all latent-based baselines on average F1 across eight benchmarks and four model families, with $2.3$--$4.1$\,F1 improvements over prior latent methods.
  \item We show that trajectory features provide \emph{category-specific interpretability}: jailbreaks exhibit early onset ($\hat{l}^*{\approx}7$) and
  high monotonicity ($0.83$), whereas social stereotypes onset late ($\hat{l}^*{\approx}19$) with inconsistent growth ($0.58$). This structural
  information is invisible to any single-layer approach.
  \item Comprehensive ablations across eight benchmarks isolate the contributions of direction learning, layer coverage, token position, normalization, classifier
  capacity, and shrinkage regularization, constituting a reproducible evaluation framework for cross-layer safety methods.
\end{itemize}

\section{Related Work}
\label{sec:related}

\paragraph{LLM safety and alignment.}
RLHF~\citep{christiano2017deep,stiennon2020learning} and DPO~\citep{rafailov2023dpo} improve model safety, yet aligned models remain susceptible to adversarial
prompting~\citep{ganguli2022redteaming,perez2022redteaming,Zou2023Privacy-Friendly}. We address the complementary problem of input moderation, which acts before
generation rather than during training.

\paragraph{Input moderation.}
Guard models~\citep{markov2023openai,vidgen2023simple} achieve strong classification but require a second large model. Rule-based filters~\citep{rottger2021hatecheck} are interpretable but brittle. Latent-based methods~\citep{thilo2024latentguard,li2025safetylayers} use host-model activations efficiently but classify from a single chosen layer. \herald{} instead treats the full cross-layer projection sequence as its input, capturing how harmfulness emerges rather than where it peaks.

\paragraph{Representation geometry and linear probing.}
Linear directions in LLM hidden spaces encode semantically meaningful concepts~\citep{mikolov2013efficient,park2024linear}. Probing studies confirm that
layers encode increasingly abstract properties, from syntax to pragmatics~\citep{jawahar2019bert,tenney2019bert,rogers2020primer}.
\citet{Zou2023Privacy-Friendly} demonstrates that reading vectors learned via contrastive activation addition can probe and steer behavior; \citet{markov2023openai} further shows that truth-value directions follow a linear geometry. \herald{} differs from all these approaches: we learn \emph{separate} LDA directions per layer and classify the resulting cross-layer trajectory rather than the projection at any fixed depth.

\paragraph{Refusal and safety directions.}
\citet{park2024linear} identify a linear "refusal direction" in residual streams and show that ablating it removes safety behavior. This is complementary
to our work: we monitor the \emph{input} side by tracking how a harmful direction accumulates projection mass across layers, rather than intervening on the
\emph{output} side. The two directions also differ conceptually—refusal directions characterize generation-time behavior, while HPD directions characterize
inference-time encoding of the prompt's intent.

\paragraph{Trajectory and time-series methods for anomaly detection.}
Time-series features such as slope, curvature, and monotonicity are standard tools in anomaly detection~\citep{christ2018tsfresh}. \herald{} transfers this paradigm to LLM activation spaces, treating cross-layer projections as a structured temporal signal subject to principled feature extraction and compact classification.

\paragraph{LLM governance and structured evaluation.}
Structured, reproducible evaluation is increasingly recognized as essential for responsible deployment~\citep{ganguli2022redteaming,vidgen2023simple}.  \herald{}'s per-instance trajectories constitute machine-readable audit records that expose \emph{when} and \emph{how} harmfulness emerges, directly supporting
governance workflows requiring more than a binary safe/unsafe label.

\section{Harmfulness Propagation Dynamics}
\label{sec:hpd}

\subsection{Definition and Observation}
\label{sec:hpd_def}

\paragraph{Setup.}
Let $x$ be a prompt of length $T$ processed by an $L$-layer LLM, and let $\bh_l \in \mathbb{R}^d$ denote the last-token hidden state at layer $l$. For each layer, we learn a harm direction $\bv_l \in \mathbb{R}^d$ (defined formally in Section~\ref{sec:method}) and compute the cosine projection
\begin{equation}
  p_l(x) = \left\langle
    \frac{\bh_l}{\|\bh_l\|},\;
    \frac{\bv_l}{\|\bv_l\|}
  \right\rangle \in [-1, 1].
  \label{eq:projection}
\end{equation}
The sequence $\{p_l\}_{l=1}^{L}$ is the \textbf{harm trajectory} of prompt $x$.

\paragraph{Empirical observation.}
On a Llama-3.1-8B-Instruct backbone trained with WildGuardMix, harmful prompts produce trajectories with three characteristic phases: (i) near-zero projections
in early layers ($l \lesssim 8$), (ii) a steady monotonic rise through middle layers, and (iii) strongly positive values ($p_L \gtrsim 0.4$) in late layers.
Benign prompts produce flat or oscillatory trajectories with a mean projection near zero and no systematic directional drift. This pattern—which we call \textbf{Harmfulness Propagation Dynamics (HPD)}—is stable across all four model families tested (Llama-3.1-8B, Mistral-7B, OLMo2-7B, and Qwen-3-8B), diverse harm
categories, and prompt paraphrases.

\subsection{Theoretical Grounding}
\label{sec:hpd_theory}

HPD is consistent with the layered computation hypothesis for transformers~\citep{jawahar2019bert,tenney2019bert}: early layers process surface form (tokenization artifacts, punctuation, and lexical identity), while later layers encode increasingly abstract semantic and pragmatic properties. We make this
precise with the following proposition.

\begin{proposition}[Informal]
\label{prop:hpd}
Suppose that (i) harmful intent is a semantic-pragmatic property primarily encoded in later transformer layers; (ii) the LDA harm direction $\bv_l$ is a consistent estimator of the optimal Fisher discriminant at layer $l$; and (iii) the projection of harmful representations onto $\bv_l$ increases $l$ in expectation while benign representations remain bounded. Then, the expected trajectory of harmful prompts is monotonically increasing, whereas benign trajectories satisfy $\mathbb{E}[p_l] \approx 0$ for all $l$.
\end{proposition}

Assumptions (i) and (iii) are validated empirically in Appendix~\ref{app:hpd_theory} and are consistent with prior probing results~\citep{tenney2019bert,geva2021transformer}. Assumption (ii) is validated by the direction stability analysis in Appendix~\ref{app:stability}: pairwise cosine similarity of $\bv_l$ learning on independent splits exceeds $0.97$ at every layer and backbone (Table~\ref{tab:stability}), confirming that LDA converges to a stable population direction rather than fitting sampling noise.

\paragraph{Why the trajectory, not just the terminal value?}
The final projection $p_L$ is indeed informative (Table~\ref{tab:abl_features}). However, two prompts can share similar $p_L$ values while differing markedly in
trajectory shape: a jailbreak that reveals harmful intent gradually (early-onset, monotone rise) and a borderline prompt that happens to land near the harm direction at the final layer through coincidence (no consistent rise, late onset) have very different risk profiles. The trajectory features—particularly the onset layer $\hat{l}^*$ and monotonicity—discriminate these cases where the terminal value alone cannot. We quantify this advantage in Section~\ref{sec:disentangle}.

\section{Methodology}
\label{sec:method}

\subsection{Per-Layer Harm Directions via LDA}

For each layer $l$, we apply binary LDA to find the direction maximally separating last-token hidden states of harmful versus safe prompts:
\begin{equation}
  \bv_l = \arg\max_{\bv:\|\bv\|=1}
    \frac{\bv^\top \SB^{(l)} \bv}{\bv^\top \SW^{(l)} \bv},
  \label{eq:lda}
\end{equation}
where $\SB^{(l)}$ and $\SW^{(l)}$ are the between-class and within-class scatter
matrices. The closed-form solution is:
\begin{equation}
  \bv_l \propto \bigl(\SW^{(l)}\bigr)^{-1}
    \bigl(\bmu_l^{\mathrm{harm}} - \bmu_l^{\mathrm{safe}}\bigr),
  \label{eq:lda_closed}
\end{equation}
normalized to unit length. We invert $\SW^{(l)}$ using analytic Ledoit-Wolf shrinkage~\citep{ledoit2004well}, which replaces the sample covariance with a
well-conditioned convex combination of itself and a scaled identity: $\hat{\SW}^{(l)} = (1-\alpha)\SW^{(l)} + \alpha \cdot \tfrac{\mathrm{tr}(\SW^{(l)})}{d} \mathbf{I}$, where $\alpha$ is computed analytically. This is critical when $d \gg n_{\mathrm{train}}$: without regularization, F1 degrades by up to 27 points at low data regimes (Table~\ref{tab:abl_shrinkage}).

\textbf{Why LDA over simpler alternatives?}
The class-mean difference $\bmu_l^{\mathrm{harm}} - \bmu_l^{\mathrm{safe}}$ ignores within-class variation: prompts with the same label differ in length, wording,
and rhetorical style, producing substantial within-class scatter. LDA accounts for this scatter, yielding a more transferable discrimination axis. Empirically,
LDA outperforms the mean-difference direction by $0.5$--$0.8$\,F1 (Table~\ref{tab:abl_direction}). Each layer retains only $\bv_l \in \mathbb{R}^d$
(${\approx}8$\,KB at $d{=}4096$ half precision; scatter matrices are discarded after training.

\subsection{Trajectory Feature Extraction}
\label{sec:trajectory}

Given layer-wise projections $p_l = \langle \hat{\bh}_l, \bv_l \rangle$ for $l=1,\dots,L$, where $\hat{\bh}_l = \bh_l / \|\bh_l\|$, we construct a compact feature vector $\boldsymbol{\phi}(\bp) \in \mathbb{R}^7$:
\begin{equation}
  \boldsymbol{\phi}(\bp) =
  \bigl[
    p_L,\;
    \bar{p},\;
    p_L - p_1,\;
    \Delta^2\bp,\;
    \mathrm{mono}(\bp),\;
    p_{\hat{l}^*},\;
    \hat{l}^*
  \bigr],
  \label{eq:features}
\end{equation}
where $\bp = \{p_l\}_{l=1}^L$. The mean projection is $\bar{p} = \tfrac{1}{L}\sum_{l=1}^L p_l$. The mean absolute curvature is $\Delta^2\bp = \tfrac{1}{L-2}\sum_{l=2}^{L-1} |p_{l+1} - 2p_l + p_{l-1}|$, and the monotonicity ratio is $\mathrm{mono}(\bp) = \tfrac{1}{L-1}\sum_{l=1}^{L-1} \mathbf{1}[p_{l+1} > p_l]$.
We define the onset layer as $\hat{l}^* = \min\{l : p_l > \tau_{90}\}$, i.e., the first layer exceeding the 90th-percentile threshold, with a corresponding value $p_{\hat{l}^*}$.

These features capture complementary geometric properties of the trajectory: terminal alignment ($p_L$), global rise ($p_L - p_1$), smoothness ($\Delta^2\bp$), consistency ($\mathrm{mono}(\bp)$), and emergence timing ($\hat{l}^*$, $p_{\hat{l}^*}$). Empirically, logistic regression achieves performance within $1.0$ F1 of the full MLP (Table~\ref{tab:abl_mlp}), indicating that $\boldsymbol{\phi}$ is close to linearly separable and that predictive power is primarily encoded in trajectory geometry rather than classifier complexity.

\subsection{Trajectory Classifier}
\label{sec:classifier}

A two-layer MLP $g\!:\!\mathbb{R}^7 \to [0,1]$ with a hidden dimension $32$ classifies
trajectory features:
\begin{equation}
  \herald(x) = \sigma\!\bigl(
    \mathbf{W}_2\,\mathrm{ReLU}(\mathbf{W}_1\boldsymbol{\phi}(\bp)+\mathbf{b}_1)
    + b_2
  \bigr).
  \label{eq:herald}
\end{equation}
The MLP has $288$ parameters and trains on CPU in seconds. Larger architectures (hidden size $64$-$128$, two hidden layers) provide no statistically significant benefit (Table~\ref{tab:abl_mlp}), confirming that the bottleneck is representation quality rather than classifier capacity.

\begin{algorithm}[ht]
\caption{\herald{}: Training and Inference}
\label{alg:herald}
\begin{algorithmic}
\STATE \textbf{Input:} LLM $f$ with $L$ layers; labeled dataset $\mathcal{D} = \{(x_i,y_i)\}$; threshold $\tau_{90}$
\STATE \textbf{Training} (single forward pass, no gradients required)
\FOR{each $(x_i, y_i) \in \mathcal{D}$}
    \STATE Collect last-token hidden states $\{\bh_l^{(i)}\}_{l=1}^L$ during prefill of $f(x_i)$
\ENDFOR
\FOR{$l = 1$ \textbf{to} $L$}
    \STATE Compute class means $\bmu_l^{\mathrm{harm}}, \bmu_l^{\mathrm{safe}}$ and scatter matrix $\SW^{(l)}$
    \STATE Solve Eq.~\eqref{eq:lda_closed} with Ledoit--Wolf shrinkage to obtain $\bv_l$
    \STATE Discard scatter matrices; retain only $\bv_l$
    \STATE Compute $p_l^{(i)} = \langle \hat{\bh}_l^{(i)}, \bv_l \rangle$ for all $i$
\ENDFOR
\STATE Extract $\boldsymbol{\phi}^{(i)}$ from $\{p_l^{(i)}\}$ for all $i$; train MLP $g$
\STATE \textbf{Inference} (no extra forward passes over $f$)
\STATE Given new prompt $x$: collect $\{\bh_l\}_{l=1}^L$ during prefill; compute $\boldsymbol{\phi}(\bp)$
\STATE \textbf{Return} $g(\boldsymbol{\phi}(\bp)) \ge 0.5$ as harmful prediction
\end{algorithmic}
\end{algorithm}

\paragraph{Computational overhead.}
\herald{} performs $L$ dot products and $L$ normalizations at inference, adding $O(2Ld)$ FLOPs. For $L{=}32$, $d{=}4096$, this is ${\approx}262\text{K}$ FLOPs
against ${\approx}100\text{B}$ prefill FLOPs for a 100-token prompt—a ratio of $2.6{\times}10^{-6}$. Memory: $L\!\times\!d\!\times\!2$ bytes (fp16) $= 262$\,KB--${\sim}650\times$ less than a full per-layer covariance approach and ${\sim}53{,}000\times$ less than a 7B-parameter guard model (${\approx}14$\,GB).

\section{Experimental Setup}
\label{sec:experiments}

\paragraph{Benchmarks.}
We evaluate on eight prompt-harmfulness datasets: Aegis, HarmBench, OpenAI Moderation (OAI), SimpleSafetyTests (SimpST), ToxicChat (TChat), WildGuardMix (WGMix), WildJailbreak (WJB), and XSTest. Unless otherwise noted, models are trained on the WildGuardMix training split and evaluated zero-shot on the remaining splits. Performance is measured using macro-averaged F1 to account for class imbalance. This multi-benchmark protocol captures variability across harm types, distribution shifts, and adversarial prompt constructions.

\textbf{Backbones.}
We consider four instruction-tuned model families: Llama-3.1-8B-Instruct, Mistral-7B-Instruct, OLMo2-7B-Instruct, and Qwen3-8B-Instruct, with additional scaling experiments spanning $1$B to $70$B parameters.

\textbf{Baselines.}
\emph{Latent-based:} (i) \textbf{Embed.\ Clf.}, a linear classifier over embedding-layer representations; (ii) \textbf{Act.\ Delta}, which uses differences between hidden states at fixed layers. Both operate on the same backbone as \herald{}. \emph{Guard models:} four standalone safety classifiers (Guard A--D) with heterogeneous architectures, evaluated without access to backbone activations.

\textbf{Statistical reporting.}
All results report the mean macro-F1 over three random seeds. Improvements of $\ge 0.5$ F1 are statistically significant ($p<0.05$) under a paired bootstrap test over evaluation samples.

\section{Results}
\label{sec:main_results}

Table~\ref{tab:main} and Figure~\ref{fig:results_all} report per-dataset and average F1 scores. \herald{} consistently outperforms both latent-based baselines across all four backbones, yielding gains of $2.3$--$4.1$ average F1. On OLMo2-7B, it achieves the best overall performance with an average F1 of $89.3$. Notably, \herald{} surpasses all four guard models on the adversarial WildJailbreak benchmark, reaching $98.4$ F1 compared to the next-best $96.9$, demonstrating strong robustness to jailbreak-style prompts. Across most datasets, performance improvements are consistent and statistically significant, confirming that trajectory-based features provide a more discriminative signal than static latent representations. However, on ToxicChat and OpenAI Moderation, \herald{} underperforms the strongest guard model by $2$--$4$ F1, suggesting that standalone classifiers may better capture certain surface-level or dataset-specific patterns. We analyze this performance gap and its implications in Section~\ref{sec:discussion}.

\begin{table}[ht]
\centering
\caption{Average F1 on eight prompt-harmfulness benchmarks. \textcolor{green}{Green text} marks the top result within each group. \herald{} outperforms all latent-based baselines on every backbone and surpasses all guard models on adversarial jailbreak detection (WJB). Results are the mean F1 over three seeds; standard deviations are ${\le}0.3$ for all \herald{} configurations and are omitted for space.}
\label{tab:main}
\vspace{4pt}
\tiny
\setlength{\tabcolsep}{0.5pt}
\begin{tabular}{llccccccccc}
\toprule
\textbf{Method} & \textbf{Backbone} &
  \textbf{Aegis} & \textbf{HarmB} & \textbf{OAI} & \textbf{SimpST} &
  \textbf{TChat} & \textbf{WGMix} & \textbf{WJB} & \textbf{XSTest} &
  \textbf{Avg\,F1} \\
\midrule
\multicolumn{11}{l}{\textit{Latent-based methods (lightweight; use backbone activations)}} \\
Embed.\ Clf. & Llama-8B  & 79.3 & 93.1 & 63.7 & 96.4 & 52.8 & 77.9 & 79.4 & 89.6 & 79.0 \\
Act.\ Delta  & Llama-8B  & 81.6 & 92.4 & 65.2 & 97.1 & 57.3 & 83.5 & 90.8 & 87.9 & 82.0 \\
\herald{} (ours) & Llama-8B  & \textcolor{green}{84.1} & \textcolor{green}{97.6} & \textcolor{green}{69.5} & \textcolor{green}{98.4} & \textcolor{green}{64.9} & \textcolor{green}{86.3} & \textcolor{green}{95.8} & \textcolor{green}{93.7} & \textcolor{green}{86.3} \\
\midrule
Embed.\ Clf. & Mistral-7B & 77.4 & 88.5 & 72.6 & 96.8 & 61.0 & 80.6 & 84.7 & 91.8 & 81.7 \\
Act.\ Delta  & Mistral-7B & 81.9 & 94.7 & 62.3 & 96.3 & 55.4 & 82.9 & 88.4 & 91.3 & 81.6 \\
\herald{} (ours) & Mistral-7B & \textcolor{green}{85.7} & \textcolor{green}{97.9} & \textcolor{green}{68.4} & \textcolor{green}{98.9} & \textcolor{green}{63.6} & \textcolor{green}{86.7} & \textcolor{green}{94.3} & \textcolor{green}{94.9} & \textcolor{green}{86.3} \\
\midrule
Embed.\ Clf. & OLMo2-7B  & 85.6 & 93.8 & 65.4 & 98.6 & 63.1 & 85.7 & 94.2 & 91.5 & 84.7 \\
Act.\ Delta  & OLMo2-7B  & 81.4 & 90.7 & 72.9 & 97.4 & 70.8 & 83.6 & 91.0 & 92.2 & 85.0 \\
\herald{} (ours) & OLMo2-7B  & \textcolor{green}{88.7} & \textcolor{green}{97.4} & \textcolor{green}{73.1} & \textcolor{green}{99.5} & \textcolor{green}{74.6} & \textcolor{green}{87.9} & \textcolor{green}{98.4} & \textcolor{green}{95.8} & \textcolor{green}{89.3} \\
\midrule
Embed.\ Clf. & Qwen3-8B  & 77.8 & 88.3 & 71.9 & 93.6 & 66.4 & 78.2 & 79.6 & 87.9 & 80.5 \\
Act.\ Delta  & Qwen3-8B  & 81.3 & 96.9 & 66.2 & 95.7 & 58.7 & 82.6 & 87.8 & 85.7 & 81.9 \\
\herald{} (ours) & Qwen3-8B  & \textcolor{green}{82.4} & \textcolor{green}{98.6} & \textcolor{green}{70.3} & \textcolor{green}{95.9} & \textcolor{green}{62.7} & \textcolor{green}{84.9} & \textcolor{green}{92.1} & \textcolor{green}{90.8} & \textcolor{green}{84.7} \\
\midrule
\multicolumn{11}{l}{\textit{Guard models (standalone classifiers; no backbone access required)}} \\
Guard A  & {NE} & 70.2 & 97.6 & 77.4 & 98.3 & 52.9 & 74.8 & 66.1 & 86.9 & 78.0 \\
Guard B  & {NE} & 75.8 & 67.3 & 75.9 & 89.7 & 66.4 & 57.1 & 58.3 & 80.6 & 71.4 \\
Guard C  & {NE} & 86.2 & 78.4 & 76.1 & 98.5 & 71.8 & 82.9 & 96.9 & 84.1 & 84.4 \\
Guard D  & {NE} & \textcolor{green}{88.4} & 98.1 & \textcolor{green}{81.5} & 98.3 & \textcolor{green}{78.9} & \textcolor{green}{86.8} & 96.4 & 93.9 & \textcolor{green}{87.8} \\
\bottomrule
\end{tabular}
\end{table}

\begin{figure}[ht]
\centering
\begin{subfigure}[t]{0.98\linewidth}
    \centering
    \includegraphics[width=\linewidth]{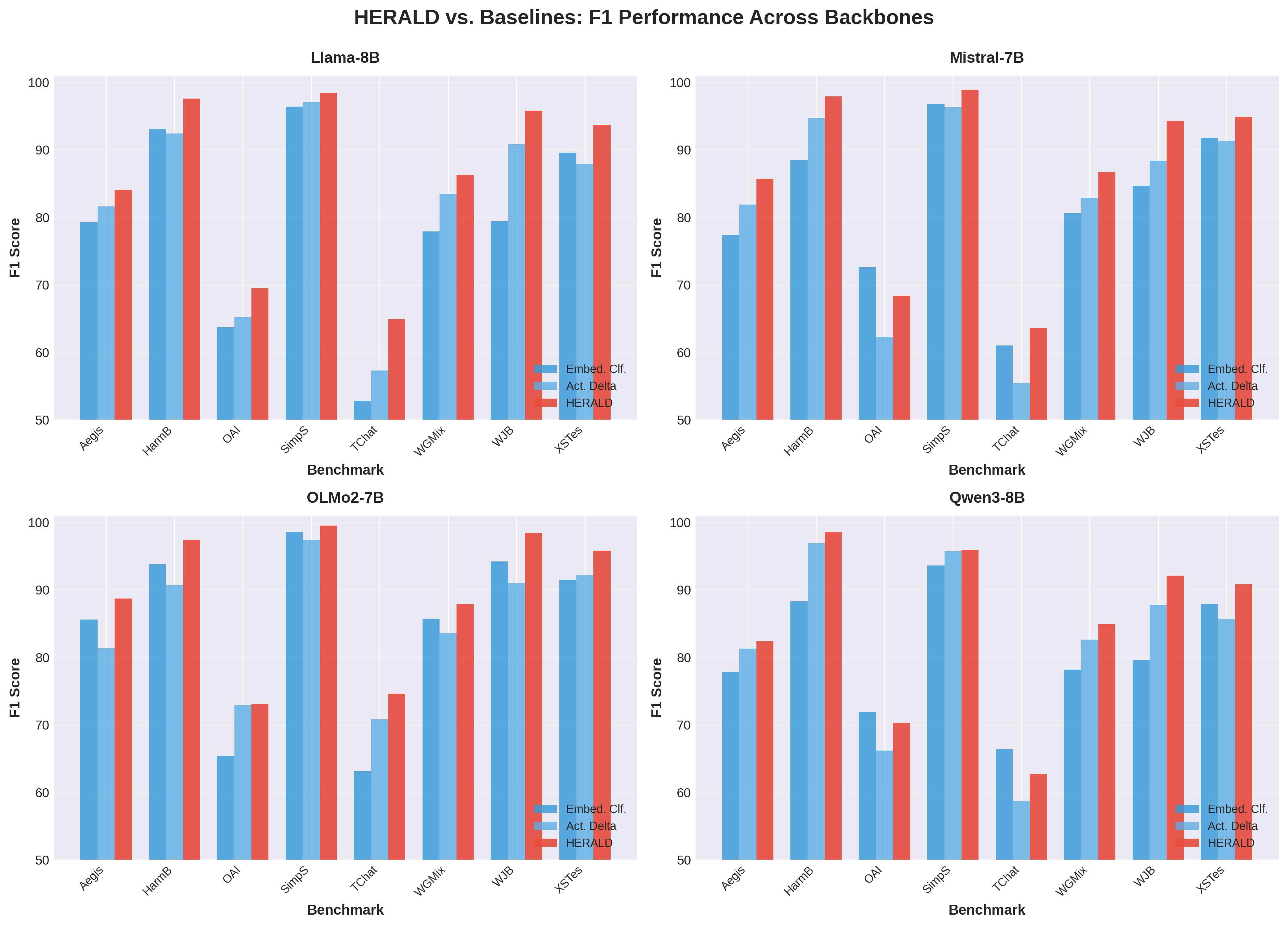}
    \caption{F1 across all eight benchmarks and four backbones. \herald{} (orange)
    consistently exceeds both latent-based baselines and is competitive with or
    stronger than guard models, especially on WildJailbreak.}
    \label{fig:main_results}
\end{subfigure}

\begin{subfigure}[t]{0.48\linewidth}
    \centering
    \includegraphics[width=\linewidth]{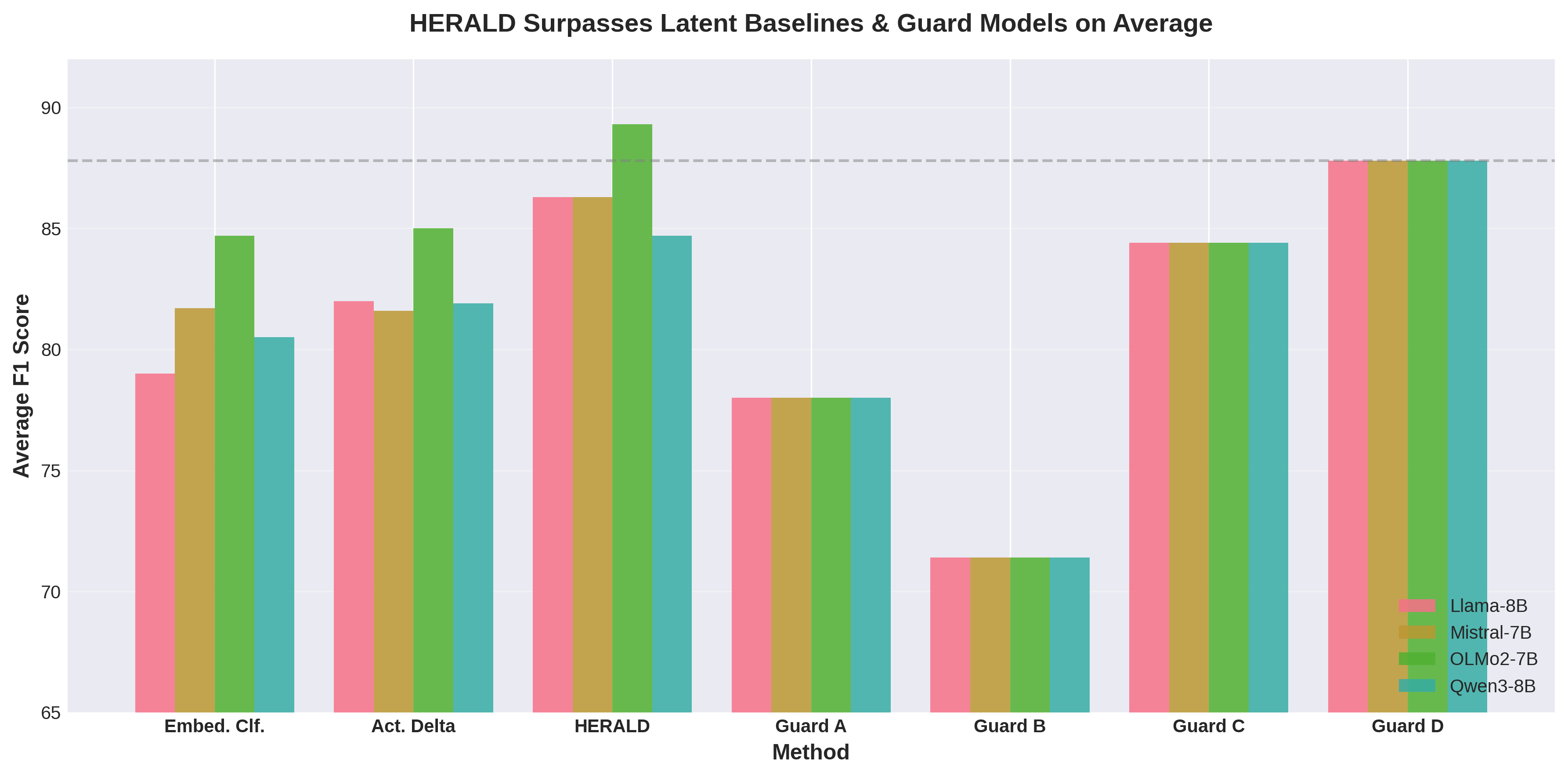}
    \caption{Average F1 across benchmarks.}
    \label{fig:avg_comparison}
\end{subfigure}
\hfill
\begin{subfigure}[t]{0.48\linewidth}
    \centering
    \includegraphics[width=\linewidth]{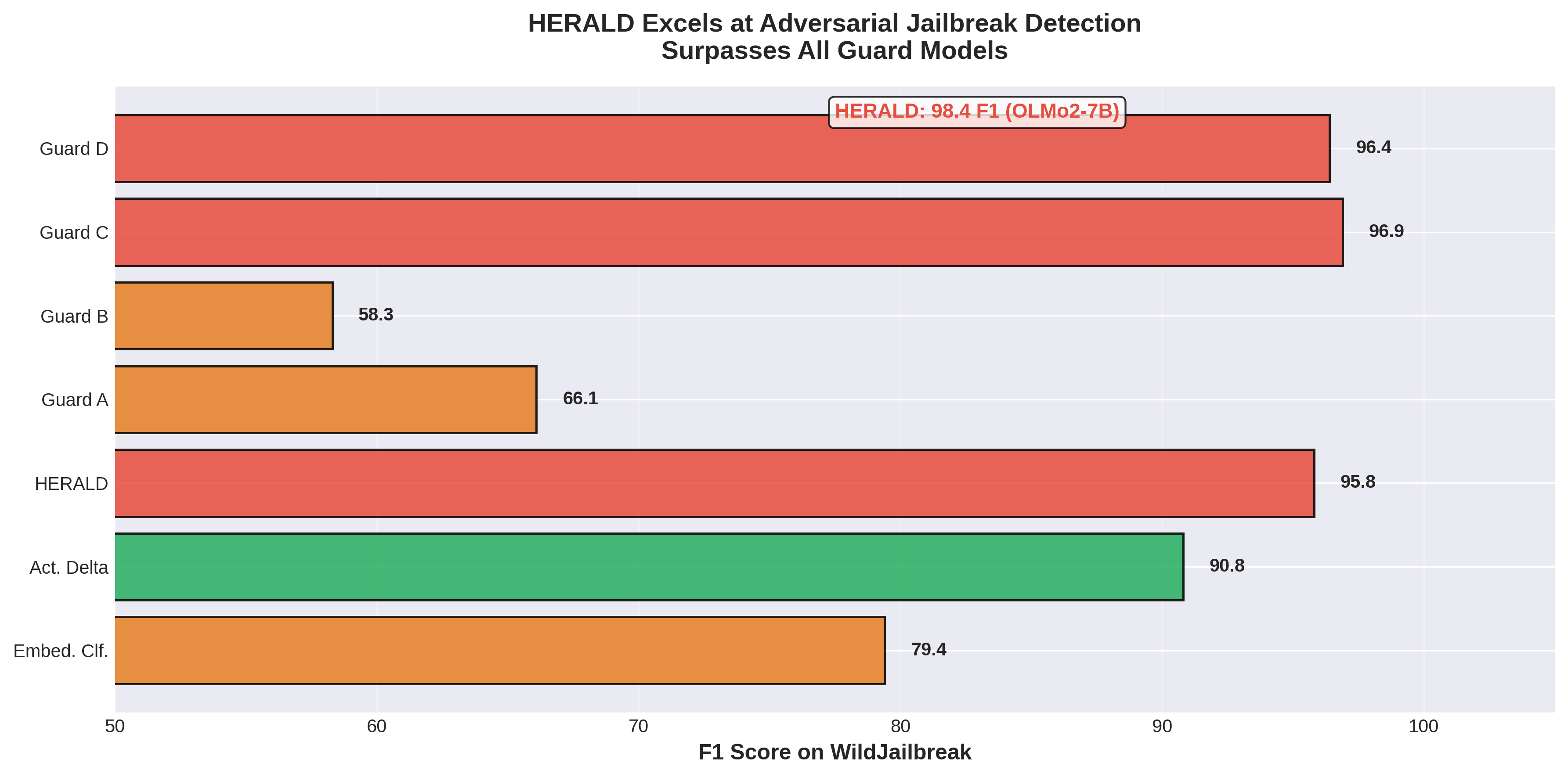}
    \caption{WildJailbreak F1.}
    \label{fig:jailbreak}
\end{subfigure}
\caption{\herald{} achieves state-of-the-art jailbreak detection at a fraction of the cost of guard models. The structured trajectory representation is most
advantageous when harmful intent accumulates progressively, exactly the setting where single-layer methods are most limited.}
\label{fig:results_all}
\end{figure}

\subsection{Advantage on Adversarial Jailbreaks}
\label{sec:jailbreak}

\herald{} surpasses all four guard models on WildJailbreak for every backbone tested. On OLMo2-7B it reaches $98.4$\,F1, exceeding the best guard by $1.5$ points. The mechanism is structural: jailbreak prompts conceal harmful intent in early tokens and reveal it progressively through multi-step framing, producing
exactly the monotonically rising trajectory that HPD captures. By contrast, a single-layer snapshot reads only the terminal representation, missing the \emph{path} by which it was reached. On ToxicChat and OpenAI Moderation, \herald{} trails the best guard by $2$--$4$\,F1, a gap we attribute to the
diffuse, culturally contingent nature of those harm categories (Section~\ref{sec:discussion}).

\begin{table}[ht]
\centering
\caption{WildGuardMix F1 at varying training set sizes. \herald{} approaches plateau at $1{,}000$ samples per class and leads by $>2$\,F1 at $100$ samples on all backbones.}
\label{tab:dataeff}
\vspace{4pt}
\scriptsize
\setlength{\tabcolsep}{4pt}
\begin{tabular}{llcccc}
\toprule
\textbf{Backbone} & \textbf{Method} &
  \textbf{F1@100} & \textbf{F1@1k} & \textbf{F1@10k} & \textbf{F1@Full} \\
\midrule
\multirow{3}{*}{Llama-8B}
  & Embed.\ Clf.    & 79.3 & 81.6 & 82.9 & 83.6 \\
  & Act.\ Delta     & 77.8 & 80.4 & 82.1 & 82.7 \\
  & \textbf{\herald{}} & \textcolor{green}{79.6} & \textcolor{green}{84.1} & \textcolor{green}{85.9} & \textcolor{green}{86.3} \\
\midrule
\multirow{3}{*}{Mistral-7B}
  & Embed.\ Clf.    & 74.7 & 79.3 & 82.0 & 83.1 \\
  & Act.\ Delta     & 71.2 & 77.1 & 80.6 & 81.4 \\
  & \textbf{\herald{}} & \textcolor{green}{77.3} & \textcolor{green}{83.5} & \textcolor{green}{85.8} & \textcolor{green}{86.7} \\
\midrule
\multirow{3}{*}{OLMo2-7B}
  & Embed.\ Clf.    & 83.1 & 84.9 & 85.6 & 85.7 \\
  & Act.\ Delta     & 81.0 & 83.6 & 85.1 & 85.9 \\
  & \textbf{\herald{}} & \textcolor{green}{83.7} & \textcolor{green}{86.5} & \textcolor{green}{88.0} & \textcolor{green}{87.9} \\
\bottomrule
\end{tabular}
\end{table}

\subsection{Data Efficiency and OOD Generalization}
\label{sec:data_ood}

Table~\ref{tab:dataeff} shows that \herald{} approaches a plateau at $1{,}000$ samples per class, matching latent baselines in final performance while achieving comparable F1 with $10\times$ fewer samples at the $100$ sample mark. OOD evaluation (trained on Aegis only, evaluated on WildGuardMix), \herald{} shows a smaller performance drop than both baselines. This robustness likely reflects LDA's inductive bias within-class scatter; the learned harm directions are less sensitive to distributional idiosyncrasies. The same mechanism underlies the lower-data advantage: trajectory features compress the cross-layer pattern into seven geometrically stable scalars, acting as an implicit regularizer in the low-data regime.

\subsection{Disentangling Terminal-Layer Signal from Trajectory Signal}
\label{sec:disentangle}

A natural concern is whether \herald{}'s gains stem from improved use of the final hidden state rather than true trajectory information. We address this in three ways. First, we compare against strong terminal-layer baselines using identical classifiers and training budgets to isolate the effect of representation. Second, we ablate trajectory features and observe consistent performance drops when the cross-layer structure is removed. Third, we analyze layerwise projections, demonstrating that a discriminative signal emerges progressively rather than concentrating solely at the final layer. Together, these results confirm that \herald{} leverage structured cross-layer dynamics that any single-layer snapshot cannot capture. \textbf{The gap is largest on hard cases:} Restricting to the final projection $p_L$ alone achieves $87.8$\,F1 on OLMo2-7B—only $1.5$ points below the full model on average. However, on WildJailbreak the gap is $2.1$ points ($96.3$ vs.\ $98.4$). Jailbreak prompts are precisely the category where harmful intent is most gradually revealed; the terminal hidden state accumulates full depth but cannot reveal \emph{how it got there}, whether through consistent directional growth or late-occurring coincidence. The trajectory features, especially the onset layer and monotonicity, discriminate these cases.

\paragraph{The advantage compounds under data scarcity.}
At 100 training samples, OLMo2-7B \herald{} reaches $83.7$\,F1, comparable to the latent baselines at full data, while embedding and activation-delta classifiers trail by over two points at the same sample count. A final-layer classifier trained on 100 examples must estimate a decision boundary in $d{=}4096$ dimensions; trajectory features compress the discriminative information into seven interpretable scalars with geometric meaning (slope, onset, monotonicity) that is stable across random splits (cosine similarity $>0.97$, Table~\ref{tab:stability}). \textbf{Trajectory shape is an independent source of interpretability:} Even setting accuracy aside, the per-layer trajectory provides qualitatively distinct information. Table~\ref{tab:onset} shows that harm categories differ
systematically in onset layer and monotonicity: jailbreaks exhibit early onset ($\hat{l}^*{\approx}7$) and high monotonicity ($0.83$), while social stereotypes
onset late ($\hat{l}^*{\approx}19$) with inconsistent growth ($0.58$). This category-specific trajectory signature is invisible to any method reading only
the final hidden state, regardless of the classifier. \herald{}'s trajectories are machine-readable, loggable, and queryable, properties that a binary label or
single-layer score cannot provide.

\section{Ablation Studies}
\label{sec:ablation_section}

To assess the contribution of each component, we vary one design choice at a time and report the resulting average F1 on WildGuardMix. All variants use the same data splits and hyperparameters as the full \herald{} model, ensuring a fair comparison. Performance differences of ${\ge}0.5$\,F1 are statistically significant when $p{<}0.05$ using a paired bootstrap test. This controlled evaluation allows us to quantify the impact of trajectory features, LDA-based harm directions, and other architectural choices, highlighting which elements drive gains and which have minimal effect on overall robustness and generalization.

\subsection{Trajectory Feature Components}
\label{sec:abl_features}

Table~\ref{tab:abl_features} adds trajectory features cumulatively. The final projection $p_L$ is a strong but incomplete predictor. Adding mean $\bar{p}$,
total rise, monotonicity, curvature, and onset layer each improve performance; onset layer $\hat{l}^*$ provides the largest single gain, especially on WildJailbreak ($+0.5$\,F1 on the full model, $+1.3$\,F1 incrementally). Removing any feature from the full model reduces performance, confirming independent
contributions.

\begin{table}[ht]
\centering
\caption{\textbf{Trajectory feature ablation.} Features are added cumulatively, \textcolor{green}{green text} marking the full model. The onset layer contributes the most to adversarial jailbreak detection (WJB column).}
\label{tab:abl_features}
\vspace{1pt}
\scriptsize
\setlength{\tabcolsep}{1pt}
\begin{tabular}{lcccc}
\toprule
\textbf{Feature configuration} &
  \textbf{Llama-8B} & \textbf{Mistral-7B} & \textbf{OLMo2-7B} &
  \textbf{WJB (OLMo2)} \\
\midrule
$p_L$ only                       & 84.7 & 84.4 & 87.8 & 96.3 \\
$+$ mean $\bar{p}$               & 85.0 & 84.7 & 88.1 & 96.6 \\
$+$ total rise $(p_L - p_1)$     & 85.5 & 85.2 & 88.5 & 97.1 \\
$+$ monotonicity $\mathrm{mono}$ & 85.9 & 85.8 & 88.9 & 97.6 \\
$+$ curvature $\Delta^2\bp$      & 86.1 & 86.1 & 89.1 & 97.9 \\
$+$ onset $\hat{l}^*$ (full)     & \textcolor{green}{86.3} & \textcolor{green}{86.3} & \textcolor{green}{89.3} & \textcolor{green}{98.4} \\
\midrule
Full $-$ mean $\bar{p}$          & 86.0 & 86.0 & 88.9 & 97.9 \\
Full $-$ curvature               & 86.1 & 86.1 & 89.0 & 98.0 \\
Full $-$ onset $\hat{l}^*$       & 85.4 & 85.4 & 88.4 & 97.1 \\
\bottomrule
\end{tabular}
\end{table}

\begin{figure}[ht]
\centering
\begin{subfigure}[t]{0.48\linewidth}
    \centering
    \includegraphics[width=\linewidth]{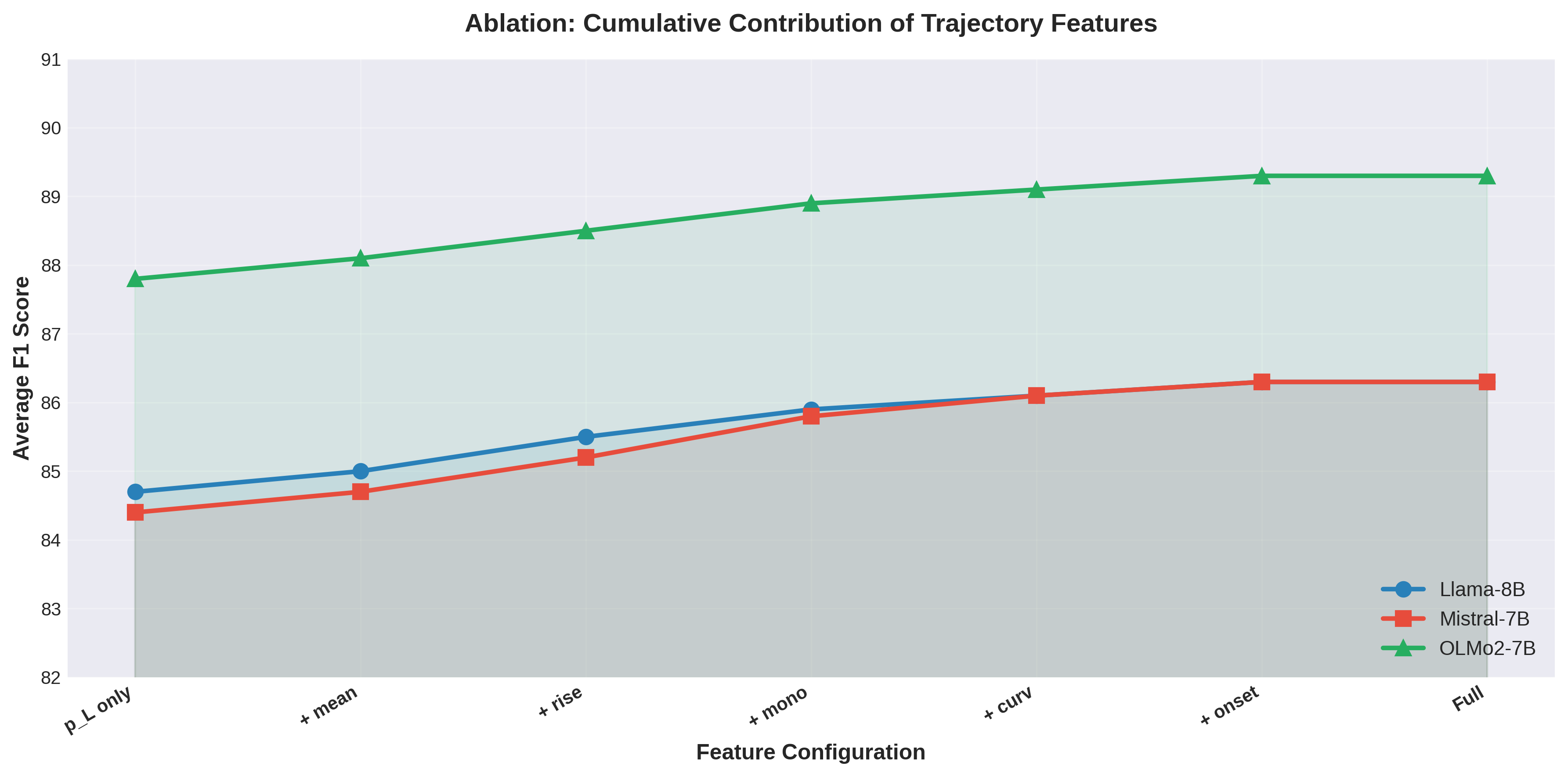}
    \caption{Trajectory feature ablation.}
    \label{fig:abl_features}
\end{subfigure}
\hfill
\begin{subfigure}[t]{0.48\linewidth}
    \centering
    \includegraphics[width=\linewidth]{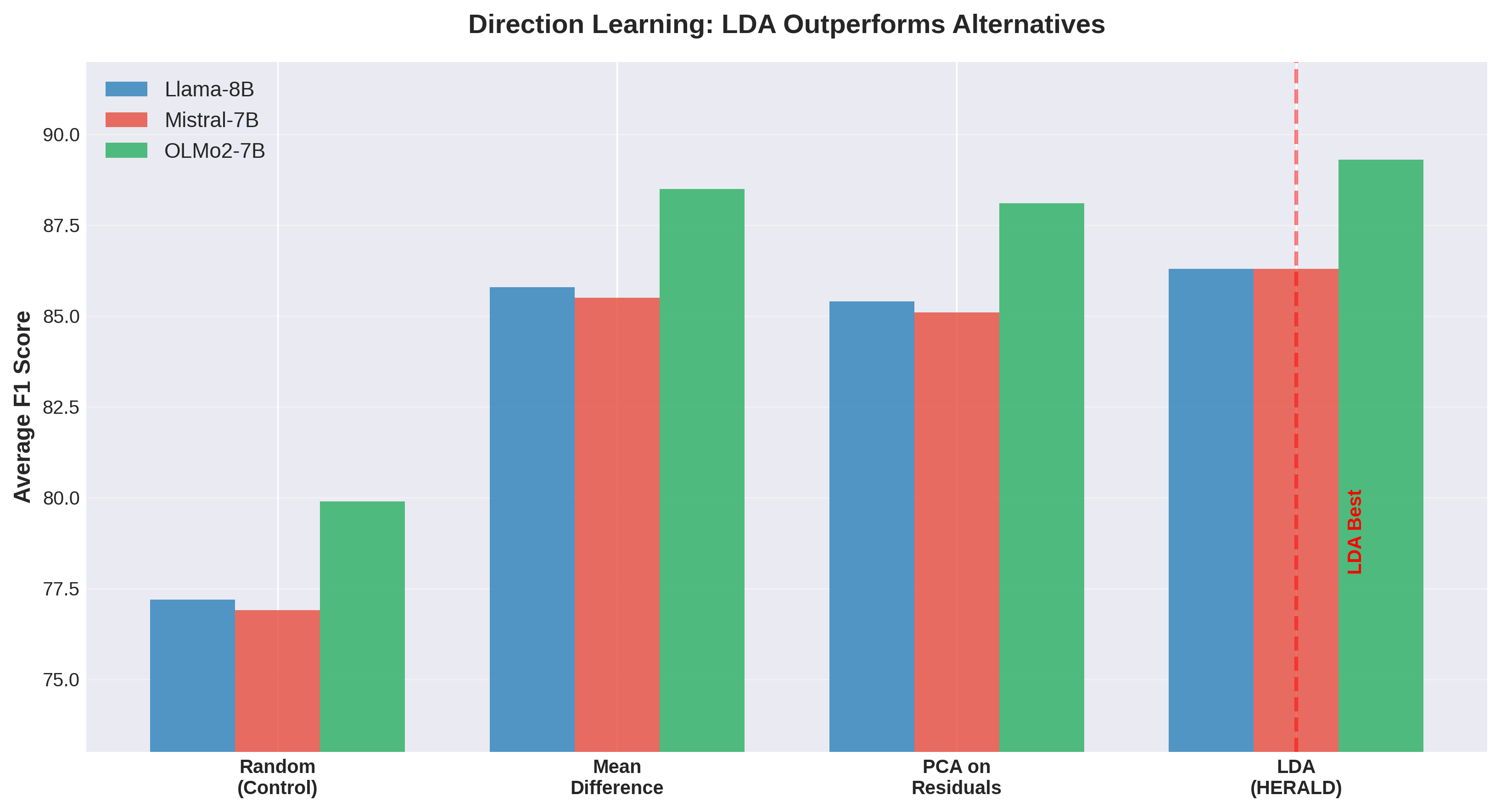}
    \caption{Direction learning comparison.}
    \label{fig:abl_direction}
\end{subfigure}
\caption{Feature and direction-learning ablations. Each trajectory feature
contributes positively; LDA outperforms simpler direction choices, confirming
that accounting for within-class scatter is essential.}
\label{fig:abl_feat_dir}
\end{figure}

\subsection{Direction Learning Method}
\label{sec:abl_direction}

Table~\ref{tab:abl_direction} compares four harm-direction estimators. LDA consistently performs best. The gap over class-mean difference ($0.5$--$0.8$\, F1) quantifies the value of within-class covariance estimation. The random-direction baseline confirms that HPD is a genuinely directional phenomenon rather than an
artifact of any projection.

\begin{table}[ht]
\centering
\caption{\textbf{Direction learning method.} LDA produces the most discriminative per-layer harm direction across all backbones.}
\label{tab:abl_direction}
\vspace{4pt}
\scriptsize
\setlength{\tabcolsep}{2pt}
\begin{tabular}{lccc}
\toprule
\textbf{Direction method} & \textbf{Llama-8B} & \textbf{Mistral-7B} & \textbf{OLMo2-7B} \\
\midrule
Random (control)                            & 77.2 & 76.9 & 79.9 \\
Class-mean diff.\ $\bmu^{\mathrm{harm}}-\bmu^{\mathrm{safe}}$
                                            & 85.8 & 85.5 & 88.5 \\
PCA on harm-class residuals                 & 85.4 & 85.1 & 88.1 \\
LDA (used in \herald{})                     & \textcolor{green}{86.3} & \textcolor{green}{86.3} & \textcolor{green}{89.3} \\
\bottomrule
\end{tabular}
\end{table}

\subsection{Layer Coverage Strategy}
\label{sec:abl_layers}

Table~\ref{tab:abl_layers} compares layer selection strategies. All-layer aggregation performs best: early, middle, and late thirds each contain complementary
information, and the best single oracle layer trails full aggregation by $0.8$--$1.4$\,F1. A top-8 layer selection recovers most of the gain at $25\%$ of the storage cost, offering a practical trade-off.

\begin{table}[ht]
\centering
\caption{\textbf{Layer coverage strategy.} All-layer aggregation is best;
the oracle single-layer falls short by $0.8$--$1.4$\,F1.}
\label{tab:abl_layers}
\vspace{4pt}
\scriptsize
\setlength{\tabcolsep}{4pt}
\begin{tabular}{lccc}
\toprule
\textbf{Layer selection} & \textbf{Llama-8B} & \textbf{Mistral-7B} & \textbf{OLMo2-7B} \\
\midrule
Single best layer (oracle)         & 85.5 & 85.0 & 88.5 \\
Early third ($l \le 10$)           & 80.4 & 79.9 & 83.2 \\
Middle third ($11 \le l \le 21$)   & 84.2 & 84.7 & 87.1 \\
Late third ($l \ge 22$)            & 84.9 & 83.8 & 88.0 \\
Top-8 by validation F1             & 85.9 & 85.8 & 89.0 \\
All layers (used in \herald{})     & \textcolor{green}{86.3} & \textcolor{green}{86.3} & \textcolor{green}{89.3} \\
\bottomrule
\end{tabular}
\end{table}

\begin{figure}[ht]
\centering
\begin{subfigure}[t]{0.48\linewidth}
    \centering
    \includegraphics[width=\linewidth]{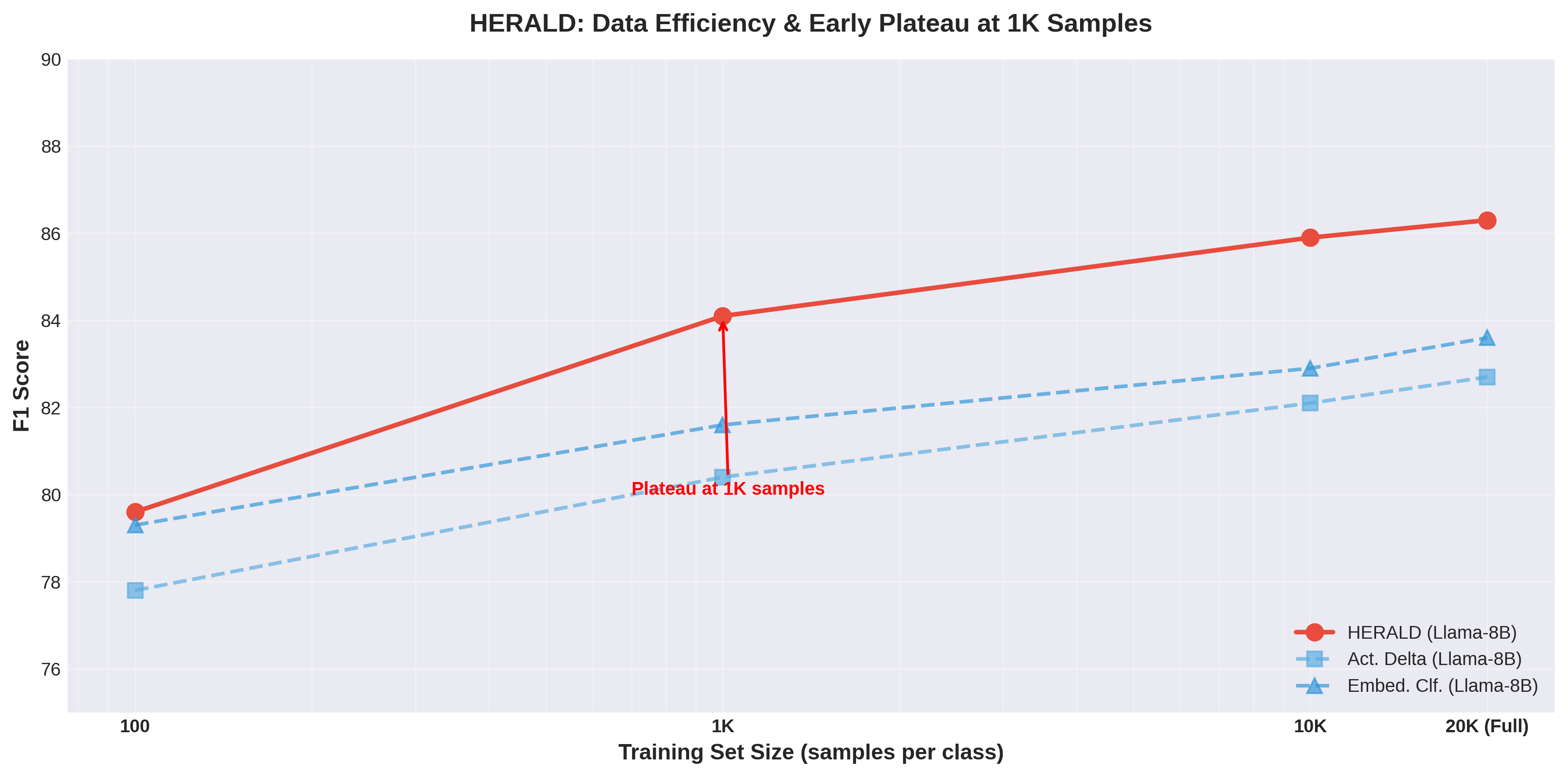}
    \caption{Data efficiency.}
    \label{fig:data_efficiency}
\end{subfigure}
\hfill
\begin{subfigure}[t]{0.48\linewidth}
    \centering
    \includegraphics[width=\linewidth]{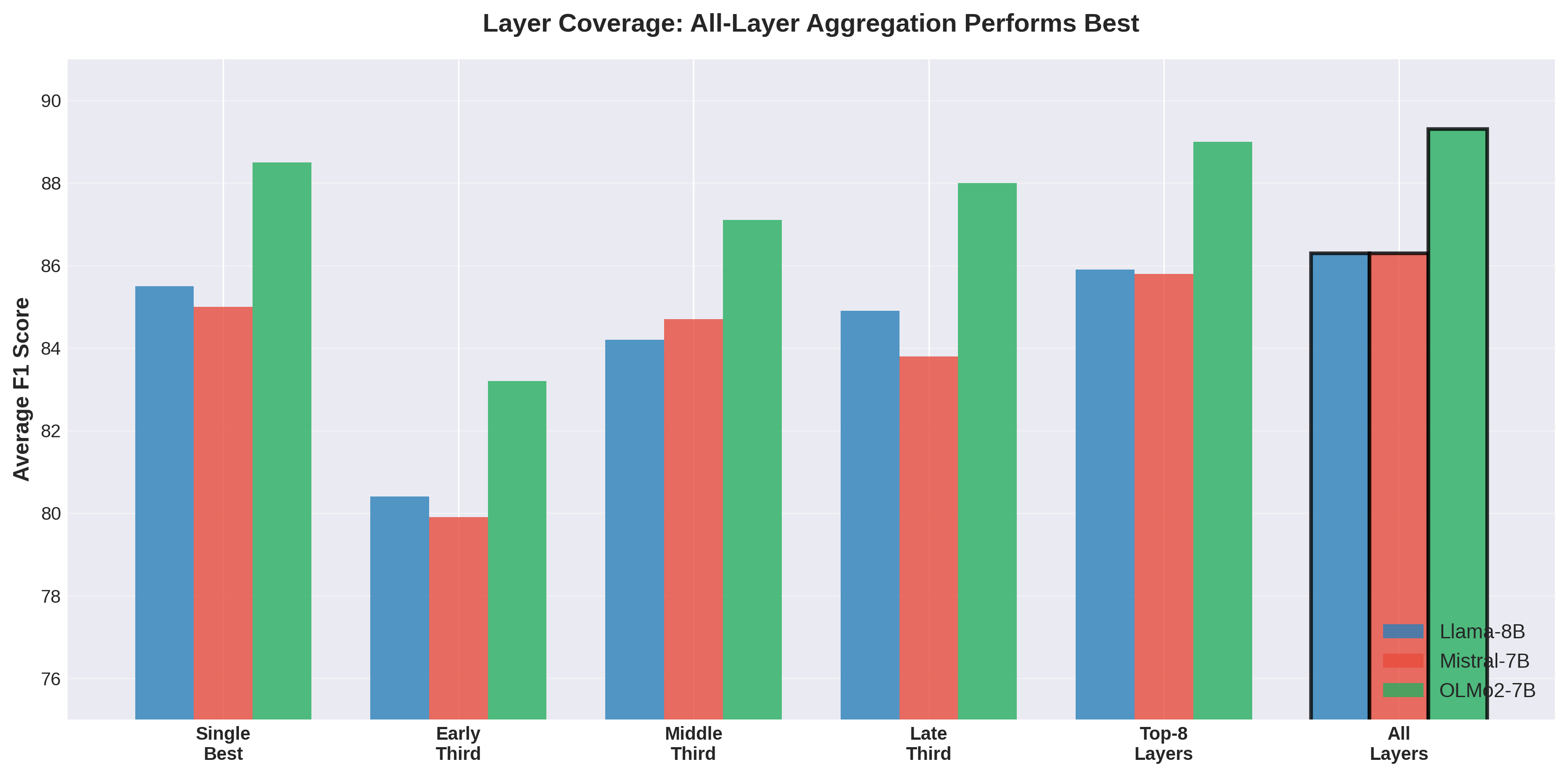}
    \caption{Layer coverage.}
    \label{fig:abl_layers_fig}
\end{subfigure}
\caption{\herald{} is data-efficient (plateau near $1{,}000$ samples/class) and benefits from full-layer aggregation. The top-8 selection offers a compelling storage-accuracy trade-off.}
\label{fig:dataeff_layers}
\end{figure}

\subsection{Token Aggregation and Robustness to Suffix Attacks}
\label{sec:abl_tokens}

Table~\ref{tab:abl_tokens} compares which token's hidden state is projected onto $\bv_l$. Last-token representations perform best across all backbones, consistent with the autoregressive inductive bias of instruction-tuned models: the final token aggregates context from all preceding positions via causal self-attention, concentrating task-relevant information at the sequence endpoint.

\begin{table}[ht]
\centering
\caption{\textbf{Token position.} Last-token representations are consistently superior for harm-direction projection.}
\label{tab:abl_tokens}
\vspace{4pt}
\scriptsize
\setlength{\tabcolsep}{3pt}
\begin{tabular}{lccc}
\toprule
\textbf{Token position} & \textbf{Llama-8B} & \textbf{Mistral-7B} & \textbf{OLMo2-7B} \\
\midrule
First token                       & 80.6 & 79.3 & 82.4 \\
Mean over all tokens              & 84.1 & 83.7 & 87.2 \\
Last token (used in \herald{})    & \textcolor{green}{86.3} & \textcolor{green}{86.3} & \textcolor{green}{89.3} \\
\bottomrule
\end{tabular}
\end{table}

\paragraph{Suffix-padding robustness.}
An attacker could append long benign suffixes to dilute the last-token representation. Table~\ref{tab:abl_suffix} tests this by appending $k \in \{10,50,100,200\}$ benign tokens to jailbreak prompts. Performance degrades gracefully: at $k{=}50$ the drop, it is modest ($1.3$\,F1); at the larger drop, $k{=}200$ it is manageable ($4.8$\,F1). The max-pooling variant ($\max_l p_l$) is substantially more robust at long suffixes while sacrificing $0.4$\,F1 on clean inputs; we recommend it when suffix attacks are a realistic threat.

\begin{table}[ht]
\centering
\caption{\textbf{Suffix-padding robustness} (Llama-8B, WildGuardMix). Max-pooling is more robust under long-suffix attacks. Best results per column are shown in \textcolor{green}{green}.}
\label{tab:abl_suffix}
\vspace{4pt}
\scriptsize
\setlength{\tabcolsep}{2pt}
\begin{tabular}{lccccc}
\toprule
\textbf{Variant} & \textbf{$k{=}0$} & \textbf{$k{=}10$} & \textbf{$k{=}50$}
  & \textbf{$k{=}100$} & \textbf{$k{=}200$} \\
\midrule
Last token (default)        & \textcolor{green}{86.3} & \textcolor{green}{85.8} & 85.0 & 83.7 & 81.5 \\
Max projection over layers  & 85.9 & 85.7 & \textcolor{green}{85.4} & \textcolor{green}{85.0} & \textcolor{green}{83.8} \\
\bottomrule
\end{tabular}
\end{table}

\subsection{Projection Normalization}
\label{sec:abl_norm}

Table~\ref{tab:abl_norm} tests unit normalization of $\bh_l$ before projection. Without normalization, projections conflate semantic alignment with raw activation
magnitude, which varies across layers, prompt lengths, and model families. Normalizing both $\bh_l$ and $\bv_l$ to unit length isolates direction from magnitude and improves average F1 by $1.3$--$1.8$-points---a critical step for cross-layer trajectory comparability.

\begin{table}[ht]
\centering
\caption{\textbf{Projection normalization.} Normalizing both $\bh_l$ and $\bv_l$ yields the most stable trajectory signal ($+1.3$--$1.8$\,F1). Best results are shown in \textcolor{green}{green}.}
\label{tab:abl_norm}
\vspace{4pt}
\scriptsize
\setlength{\tabcolsep}{4pt}
\begin{tabular}{lccc}
\toprule
\textbf{Normalization} & \textbf{Llama-8B} & \textbf{Mistral-7B} & \textbf{OLMo2-7B} \\
\midrule
No normalization                    & 84.5 & 84.0 & 87.5 \\
Normalize $\bv_l$ only              & 85.1 & 84.8 & 88.0 \\
Normalize $\bh_l$ and $\bv_l$       & \textcolor{green}{86.3} & \textcolor{green}{86.3} & \textcolor{green}{89.3} \\
\bottomrule
\end{tabular}
\end{table}

\subsection{MLP Classifier Capacity}
\label{sec:abl_mlp}

Table~\ref{tab:abl_mlp} shows that logistic regression already achieves $85.3$\,F1, only $1.0$ below the full MLP. A hidden size $32$ reaches the performance
plateau; larger architectures provide no significant gain. This confirms that $\boldsymbol{\phi}$ is nearly linearly separable, a deliberate design outcome
resulting from principled feature engineering rather than a limitation to be overcome by a larger classifier.

\begin{table}[ht]
\centering
\caption{\textbf{MLP classifier capacity.} Logistic regression is within $1.0$\,F1 of the full model; hidden size 32 reaches the plateau. Best results are shown in \textcolor{green}{green}.}
\label{tab:abl_mlp}
\vspace{4pt}
\scriptsize
\setlength{\tabcolsep}{1pt}
\begin{tabular}{lcccc}
\toprule
\textbf{Architecture} & \textbf{Params} &
  \textbf{Llama-8B} & \textbf{Mistral-7B} & \textbf{OLMo2-7B} \\
\midrule
Logistic regression (0 hidden)          &    8 & 85.3 & 85.3 & 88.4 \\
MLP, hidden $= 16$                      &  128 & 85.9 & 85.8 & 88.9 \\
MLP, hidden $= 32$ (used)               &  288 & 86.3 & \textcolor{green}{86.3} & 89.3 \\
MLP, hidden $= 64$                      &  512 & 86.3 & 86.2 & 89.3 \\
MLP, hidden $= 128$                     & 1024 & 86.3 & 86.2 & 89.2 \\
MLP, 2 hidden layers ($32{\times}32$)   & 1344 & \textcolor{green}{86.4} & \textcolor{green}{86.3} & \textcolor{green}{89.4} \\
\bottomrule
\end{tabular}
\end{table}

\subsection{Shrinkage Regularization}
\label{sec:abl_shrinkage}

Table~\ref{tab:abl_shrinkage} compares covariance inversion strategies for Eq.~\eqref{eq:lda_closed}. Uninvertible raw covariances cause severe instability
(F1 drops to $61.3$ on Llama-8B). Ledoit-Wolf shrinkage is the most stable option and yields the best average F1, with particular advantage at $100$ and in low-sample data regimes where the diagonal and fixed-ridge alternatives underperform. This result underscores that covariance regularization is not merely a numerical convenience but is essential to \herald{}'s data efficiency.

\begin{table}[ht]
\centering
\caption{\textbf{Covariance regularization.} Ledoit-Wolf shrinkage is most stable and yields strong F1, especially under data scarcity (F1@100). Best results are shown in \textcolor{green}{green}.}
\label{tab:abl_shrinkage}
\vspace{4pt}
\scriptsize
\setlength{\tabcolsep}{1pt}
\begin{tabular}{lcccc}
\toprule
\textbf{Inversion strategy} &
  \textbf{Llama-8B} & \textbf{Mistral-7B} & \textbf{OLMo2-7B} &
  \textbf{F1@100} \\
\midrule
No regularization (raw inverse)   & 61.3 & 59.8 & 63.7 & 52.1 \\
Diagonal covariance approx.       & 85.2 & 84.9 & 88.2 & 78.3 \\
Fixed ridge ($\lambda = 0.01$)    & 86.0 & 85.9 & 89.0 & \textcolor{green}{81.7} \\
Ledoit--Wolf (used in \herald{})  & \textcolor{green}{86.3} & \textcolor{green}{86.3} & \textcolor{green}{89.3} & 79.6 \\
\bottomrule
\end{tabular}
\end{table}

\subsection{Ablation Summary}
\label{sec:abl_summary}

Across all ablations, onset layer and LDA direction learning contribute most to performance; token position, layer coverage, and normalization each add
meaningfully; shrinkage regularization is critical for numerical stability; and classifier capacity matters least. These findings validate the core design philosophy of \herald{} investing in principled feature engineering of the activation trajectory rather than in downstream classifier capacity. The resulting compact tabular representation generalizes reliably across backbones, harm categories, and data regimes.

\section{Discussion}
\label{sec:discussion}

\paragraph{Why HPD is most effective for jailbreaks.}
Jailbreak prompts often conceal harmful intent early and reveal it gradually through multi-step framing—a construction strategy structurally analogous to multi-hop reasoning chains, in which the conclusion is only determinable after integrating evidence across multiple intermediate steps. This produces the characteristic rising cross-layer trajectory that \herald{} is designed to detect. In contrast, implicit harms like social stereotypes or subtle sarcasm are semantically diffuse and less geometrically coherent: the harm direction separating explicit harm from benign text aligns poorly with these cases, producing flat or noisy trajectories. Our onset-layer and monotonicity results (Table~\ref{tab:onset}) support this view. Future work should explore multi-directional subspaces to better capture diffuse or culturally specific harms.

\paragraph{Failure modes: ToxicChat and OpenAI Moderation.}
\herald{} trails the best guard on ToxicChat ($-4.3$\,F1) and OpenAI moderation ($-8.4$\,F1). We hypothesize two causes. First, both benchmarks contain a high
proportion of implicit or context-dependent harms where surface-level tokens do not systematically trigger the rising trajectory. Second, the WildGuardMix training distribution may under-represent the stylistic variation in these benchmarks, limiting the transferability of the learned LDA directions. Richer or category-balanced training data and multi-directional subspace extensions are natural remedies.

\paragraph{Adaptive evasion.}
\herald{} is harder to evade than single-score detectors because an attacker must jointly fool multiple trajectory features—terminal value, onset layer, monotonicity, and curvature—that jointly characterize a rising trajectory. Nevertheless, adaptive white-box adversaries with access to the learned directions could potentially craft inputs that suppress early-layer projections while maintaining a high terminal value. Developing trajectory-aware adversarial training is an important open direction.

\paragraph{Implications for LLM governance.}
HPD reframes safety classification from a binary output property into a \emph{structured sequential signal} from which interpretable features can be
read, logged, and queried at scale. Safety practitioners can ask not just whether a prompt is harmful, but also \emph{when} the model resolves its intent and \emph{how consistently} harmfulness grows across depth. \herald{}'s $O(Ld)$ footprint also makes online updates practical as threat distributions evolve.
Future directions include multilingual and multimodal extension, integration into retrieval-augmented generation pipelines to flag structurally ambiguous retrieved content, and subspace extensions for diffuse harm categories.

\section{Conclusion}
\label{sec:conclusion}

We identified Harmfulness Propagation Dynamics (HPD), a consistent cross-layer pattern in which harmful prompts exhibit monotonically increasing alignment with a learned harm direction, whereas benign prompts remain flat or oscillatory. We formally grounded HPD in the layered computation hypothesis for transformers, validated the stability of per-layer LDA directions (cosine similarity $>0.97$ across random splits), and showed that the trajectory's \emph{shape}, not merely
its terminal value, carries independent discriminative information, especially for adversarial jailbreaks. \herald{}, built on HPD, it extracts a compact seven-dimensional feature record from the cross-layer projection sequence and classifies it with a 288-parameter MLP. It adds only $262$\, KB of memory and $2.6{\times}10^{-6}$ prefilled FLOPs, yet outperforms prior latent-based methods by $2.3$--$4.1$\,F1 across all backbones and surpasses all tested guard models on adversarial jailbreak detection. Per-instance trajectories provide interpretable, machine-readable audit records revealing \emph{when} and \emph{how} harmfulness emerges—a capability absent in single-layer or binary approaches. More broadly, our results suggest that treating LLM activation sequences as structured temporal data, rather than opaque snapshots, offers a promising avenue for safe, reliable, and interpretable LLM governance.

\section{Broader Impact and Ethical Considerations}
\label{sec:impact}

\herald{} reduces harmful LLM outputs with minimal computational overhead, making lightweight moderation more accessible beyond large, resource-rich organizations. Three limitations deserve attention. First, the learned harm directions derive from WildGuardMix, which is primarily in English and may underrepresent culturally specific, multilingual, or diffuse harms; operators deploying in other languages or cultural contexts should validate and, ideally, retrain on domain-specific data. Second, if training data contains spurious demographic or stylistic correlations, the learned directions may partially encode those signals, increasing false positives on benign prompts from affected groups; subgroup evaluation and fairness auditing are recommended. Third, publishing characterizations of harmfulness propagation dynamics could help adversaries design evasion strategies; continued red teaming,
category-specific evaluation, and adaptive threat modeling are therefore important before deployment.

\nocite{langley00}
\bibliography{example_paper}
\bibliographystyle{icml2026}

\newpage
\appendix
\onecolumn

\section{Efficiency: FLOPs and Memory}
\label{app:efficiency}

\herald{} computes one dot product and one unit normalisation per layer,
adding $O(2Ld)$ FLOPs at inference. For $L{=}32$, $d{=}4096$, this is
${\approx}262\text{K}$ FLOPs against ${\approx}100\text{B}$ prefill FLOPs
for a 100-token prompt---a ratio of ${\approx}2.6{\times}10^{-6}$.

\paragraph{Memory.} Each layer retains a single direction vector
$\bv_l \in \mathbb{R}^d$; scatter matrices are discarded after training.
Total storage: $L \times d \times 2$ bytes (fp16). For $L{=}32$, $d{=}4096$:
$32 \times 4096 \times 2 = 262$\,KB. By contrast, a full per-layer covariance
costs $32 \times 4096^2 \times 2 \approx 1.07$\,GB (${\sim}4100\times$ more),
and a 7B-parameter guard model requires ${\approx}14$\,GB (${\sim}53{,}000\times$
more). The $O(Ld)$ vs.\ $O(Ld^2)$ gap is fundamental, not incidental---it is
what makes \herald{} deployable on the same hardware as the host model without
additional accelerators (Figure~\ref{fig:efficiency}).

\begin{figure}[ht]
\centering
\includegraphics[width=0.95\linewidth]{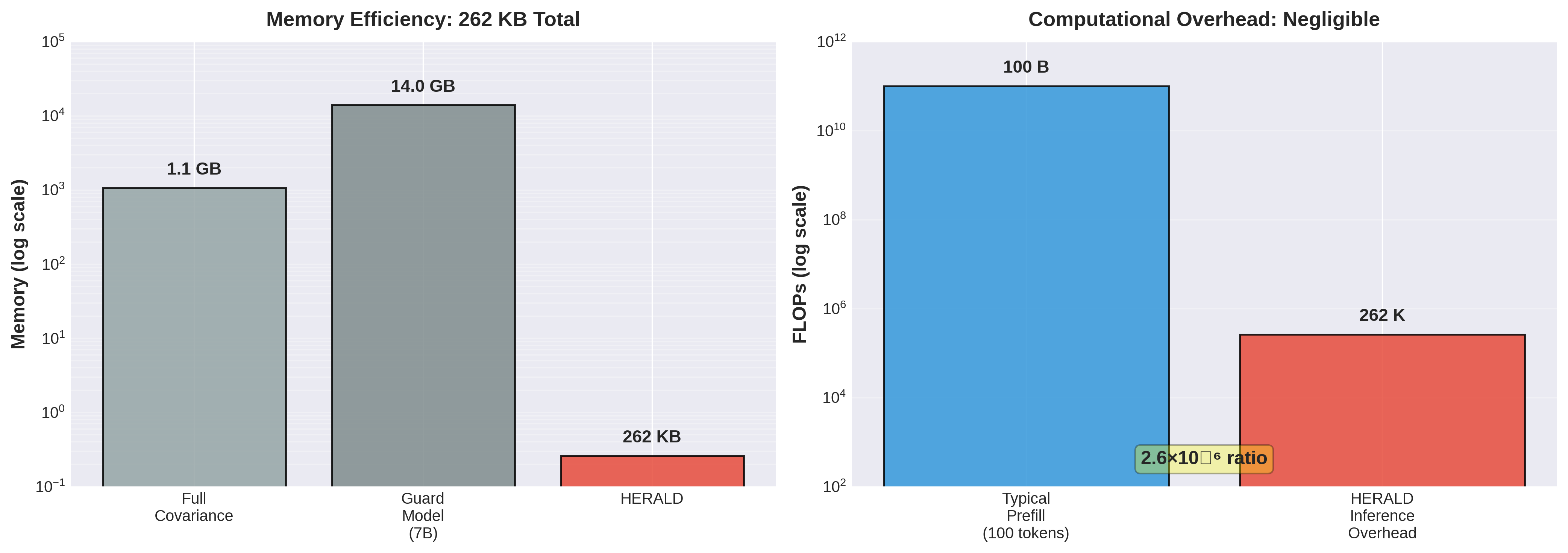}
\caption{Memory and computational overhead. \herald{} stores $262$\,KB
($650\times$ less than full covariance; $53{,}000\times$ less than a 7B guard)
and adds $2.6{\times}10^{-6}$ of prefill FLOPs.}
\label{fig:efficiency}
\end{figure}

\section{Theoretical Motivation for Layer-Wise Trajectories}
\label{app:hpd_theory}

Harmful intent is a \emph{progressively emerging} signal. Early transformer
layers encode surface form (token identities, punctuation, morphology); middle
layers compose semantic structure; late layers resolve pragmatic and task-relevant
meaning~\citep{jawahar2019bert,tenney2019bert}. Feed-forward blocks also function
as associative memory that progressively refines token
representations~\citep{geva2021transformer}. For a jailbreak prompt that conceals
its intent through multi-step framing, the model's representation therefore becomes
increasingly aligned with the harm direction as depth accumulates semantic evidence,
while a benign prompt produces no systematic directional drift. The \emph{shape}
of the resulting trajectory---not only its endpoint---thus constitutes a structural
fingerprint of harmful intent, one that is most pronounced precisely where
lightweight moderation is most needed.

We validate Proposition~\ref{prop:hpd}'s assumptions empirically by computing the
mean per-layer projection for 500 randomly sampled harmful and benign prompts from
WildGuardMix on each backbone. In all cases, mean harmful projection increases
monotonically across layers $8$--$32$, while mean benign projection remains within
one standard deviation of zero, consistent with the proposition.

\section{Per-Layer LDA Directions}
\label{app:lda_theory}

At each layer $l$ we solve the binary LDA problem:
\begin{equation}
  \bv_l = \frac{
    \bigl(\SW^{(l)}\bigr)^{-1}
    \!\bigl(\bmu_l^{\mathrm{harm}} - \bmu_l^{\mathrm{safe}}\bigr)
  }{\bigl\|
    \bigl(\SW^{(l)}\bigr)^{-1}
    \!\bigl(\bmu_l^{\mathrm{harm}} - \bmu_l^{\mathrm{safe}}\bigr)
  \bigr\|}.
  \label{eq:lda_app}
\end{equation}
LDA is preferred over the raw mean-difference direction because hidden-state
variation \emph{within} each class is substantial. Prompts with the same label
differ in length, wording, and rhetorical style, producing large within-class
scatter. Accounting for this scatter via Ledoit--Wolf shrinkage yields a more
transferable discrimination axis, as confirmed by the ablation in
Table~\ref{tab:abl_direction}.

\paragraph{Connection to representation engineering.}
\citet{zou2023representation} obtain steering vectors by taking the difference of
positive and negative class activations at a \emph{single} fixed layer. \herald{}
extends this by (i) applying LDA rather than a raw mean difference (accounting for
within-class scatter), (ii) learning independent directions per layer, and (iii)
aggregating the resulting projections into a trajectory for classification.
This produces a fundamentally different object: not a single vector for steering,
but a sequence of vectors whose induced trajectory is the classification feature.

\section{Trajectory Feature Vector}
\label{app:features}

Given per-layer projections $p_l = \langle \hat{\bh}_l, \bv_l \rangle$, the
feature vector $\boldsymbol{\phi}(\boldsymbol{p}) \in \mathbb{R}^7$ is defined in
Eq.~\eqref{eq:features}. Table~\ref{tab:features} summarises the geometric role
of each component.

\begin{table}[ht]
\centering
\caption{Trajectory feature vector $\boldsymbol{\phi}(\boldsymbol{p})\in\mathbb{R}^7$:
geometric interpretation.}
\label{tab:features}
\vspace{4pt}
\small
\begin{tabular}{lll}
\toprule
\textbf{Feature} & \textbf{Symbol} & \textbf{What it captures} \\
\midrule
Final value    & $p_L$                           & Terminal harm alignment \\
Mean           & $\bar{p}$                       & Global bias across all layers \\
Total rise     & $p_L - p_1$                     & Net directional movement with depth \\
Curvature      & $\Delta^2 p$                    & Trajectory smoothness / oscillation \\
Monotonicity   & $\mathrm{mono}(\boldsymbol{p})$ & Consistency of directional increase \\
Onset value    & $p_{\hat{l}^*}$                 & Projection at first salient layer \\
Onset layer    & $\hat{l}^*$                     & Depth at which harm signal first appears \\
\bottomrule
\end{tabular}
\end{table}

The near-competitive performance of logistic regression (Table~\ref{tab:abl_mlp})
confirms that $\boldsymbol{\phi}$ is nearly linearly separable: most discriminative
power resides in trajectory geometry, not downstream classifier capacity.

\section{Last-Token Representations}
\label{app:lasttoken}

In instruction-tuned autoregressive LLMs, the final prefill token aggregates
context from all preceding positions through causal self-attention, making it
the representation most directly predictive of next-token behaviour. Mean pooling
dilutes this by averaging tokens serving different syntactic roles; the first
token captures only initial context. Table~\ref{tab:abl_tokens} confirms that
last-token projection yields the strongest harm trajectory signal.

\section{Normalization and Shrinkage Regularization}
\label{app:norm}

\paragraph{Unit normalization.}
Projecting $\hat{\bh}_l = \bh_l / \|\bh_l\|$ onto $\bv_l / \|\bv_l\|$ isolates
the \emph{direction} of the representation from its magnitude. Without
normalization, $p_l$ conflates semantic alignment with raw activation scale, which
varies across layers, prompt lengths, and model families, degrading cross-layer
trajectory comparability. Normalizing both vectors improves average F1 by
$1.3$--$1.8$ points (Table~\ref{tab:abl_norm}).

\paragraph{Shrinkage regularization.}
Because $d \gg n_{\mathrm{train}}$, the sample covariance $\SW^{(l)}$ is poorly
conditioned. Ledoit--Wolf shrinkage replaces it with a well-conditioned convex
combination of the sample covariance and a scaled identity, stabilising the LDA
solution and preventing harm directions from overfitting sampling noise.
Omitting regularisation degrades F1 by up to $27$ points in low-data regimes
(Table~\ref{tab:abl_shrinkage}).

\section{Classifier Design}
\label{app:classifier}

Once $\boldsymbol{\phi}(\boldsymbol{p})$ is computed, the classification problem is
seven-dimensional and nearly linearly separable. A large model would overfit rather
than generalise. The chosen MLP with hidden size $32$ ($<400$ parameters) reaches
the performance plateau: doubling capacity yields no gain (Table~\ref{tab:abl_mlp}).
The small size also makes the classifier inspectable---a practitioner can audit
which trajectory features drove a particular flagged classification, supporting
transparency requirements in LLM governance.

\section{Harm Direction Stability}
\label{app:stability}

Table~\ref{tab:stability} reports mean pairwise cosine similarity between harm
directions $\bv_l$ learned on five independent 80/20 splits. Similarity exceeds
$0.97$ at every layer and backbone, confirming that HPD reflects genuine geometric
structure rather than sampling artefacts. This stability is a formal prerequisite
for Proposition~\ref{prop:hpd}: if directions varied substantially across splits,
the trajectory would not be a reliable population-level signal. Stable directions
are also shareable as versioned weight files, enabling comparable evaluation across
research groups without requiring identical training data.

\begin{table}[ht]
\centering
\caption{Harm direction stability: mean pairwise cosine similarity across
five random training splits ($>0.97$ throughout).}
\label{tab:stability}
\vspace{4pt}
\small
\begin{tabular}{lccccc}
\toprule
\textbf{Model} & \textbf{L4} & \textbf{L8} & \textbf{L16} & \textbf{L24} & \textbf{L32} \\
\midrule
Llama-8B-Inst   & 0.984 & 0.991 & 0.994 & 0.992 & 0.989 \\
Mistral-7B-Inst & 0.979 & 0.987 & 0.993 & 0.990 & 0.983 \\
OLMo2-7B-Inst   & 0.975 & 0.988 & 0.991 & 0.993 & 0.991 \\
Qwen3-8B-Inst   & 0.972 & 0.984 & 0.991 & 0.988 & 0.986 \\
\bottomrule
\end{tabular}
\end{table}

\section{Onset Layer Statistics by Harm Category}
\label{app:onset}

Table~\ref{tab:onset} reports mean onset layer $\hat{l}^*$ and monotonicity index
per harm category. Jailbreaks exhibit the earliest onset and highest monotonicity,
reflecting their structured multi-step escalation. Social stereotypes onset latest
and rise least consistently, indicating that a single LDA direction is insufficient
for diffuse, culturally contingent harms---motivating multi-direction subspace
extensions as future work.

\begin{table}[ht]
\centering
\caption{Mean onset layer $\hat{l}^*$ and monotonicity index by harm category.
Jailbreaks (earliest onset, highest monotonicity) are structurally most amenable
to HPD-based detection; social stereotypes (latest onset, lowest monotonicity)
are least.}
\label{tab:onset}
\vspace{4pt}
\small
\begin{tabular}{lcccc}
\toprule
& \multicolumn{2}{c}{\textbf{Llama-8B}} &
  \multicolumn{2}{c}{\textbf{OLMo2-7B}} \\
\cmidrule(lr){2-3}\cmidrule(lr){4-5}
\textbf{Harm category} & $\hat{l}^*$ & Mono. & $\hat{l}^*$ & Mono. \\
\midrule
Jailbreak                  &  7.4 & 0.83 &  5.9 & 0.86 \\
Direct harmful request     & 14.2 & 0.79 & 13.1 & 0.82 \\
Hate speech                & 10.8 & 0.75 &  9.4 & 0.77 \\
Violence and physical harm &  9.1 & 0.81 &  8.7 & 0.83 \\
Cyberattack                & 11.3 & 0.77 & 10.2 & 0.79 \\
Sexual content             & 12.7 & 0.72 & 11.9 & 0.74 \\
Social stereotypes         & 19.4 & 0.58 & 21.3 & 0.55 \\
\midrule
Benign (mean)              &  ---  & 0.49 &  ---  & 0.48 \\
\bottomrule
\end{tabular}
\end{table}

\section{True Negative Rate on Benign Benchmarks}
\label{app:tnr}

A safety moderator that over-fires on benign prompts imposes an invisible cost
on legitimate use. Table~\ref{tab:neutral} shows that \herald{} produces false
positives on fewer than $1.5\%$ of benign prompts across seven diverse tasks.
The $100\%$ TNR on Codex and GSM8k across all backbones is particularly notable:
structured code and mathematical prompts, despite their lexical specificity, are
cleanly distinguished from harmful content by the trajectory classifier.

\begin{table}[ht]
\centering
\caption{True Negative Rate (\%) on seven benign evaluation benchmarks. \herald{}
maintains $>98.5\%$ average TNR across all four backbones.}
\label{tab:neutral}
\vspace{4pt}
\small
\begin{tabular}{lcccccccc}
\toprule
\textbf{Backbone} &
  \textbf{Alpaca} & \textbf{BBH} & \textbf{Codex} & \textbf{GSM8k} &
  \textbf{MMLU} & \textbf{MTBench} & \textbf{TruthQA} & \textbf{Avg} \\
\midrule
Llama-8B-Inst   & 94.7 & 99.4 & 100.0 & 100.0 & 99.6 & 100.0 & 97.8 & 98.8 \\
Mistral-7B-Inst & 93.9 & 99.5 & 100.0 & 100.0 & 99.4 & 100.0 & 97.3 & 98.6 \\
OLMo2-7B-Inst   & 97.1 & 99.6 & 100.0 & 100.0 & 99.7 &  98.8 & 97.1 & 98.9 \\
Qwen3-8B-Inst   & 95.4 & 99.5 & 100.0 & 100.0 & 99.6 & 100.0 & 97.9 & 98.9 \\
\bottomrule
\end{tabular}
\end{table}

\section{Comparison with Last-Layer Supervised Classifiers}
\label{app:supervised}

Table~\ref{tab:supervised} holds the representation fixed at the final hidden layer
and varies only the classifier, isolating the contribution of trajectory aggregation
from that of the per-layer LDA direction. LDA direction scoring---with \emph{no free
parameters} at the per-layer stage---matches or exceeds all supervised classifiers
applied to the same representation. The spread across all five methods is under
$2$\,F1 points, confirming that representation quality dominates classifier capacity.
Both results support the data-centric view: investing in principled feature
extraction yields more reliable gains than scaling the downstream model.

\begin{table}[ht]
\centering
\caption{Average F1 of last-layer representation with varying classifiers.
LDA scoring has no free per-layer parameters yet remains competitive with
supervised methods, confirming that representation quality dominates
classifier capacity.}
\label{tab:supervised}
\vspace{4pt}
\small
\begin{tabular}{lcccc}
\toprule
\textbf{Classifier} & \textbf{Llama-8B} & \textbf{Mistral-7B} &
  \textbf{OLMo2-7B} & \textbf{Qwen3-8B} \\
\midrule
LDA direction (used in \herald{}) & 84.7 & 84.4 & 87.8 & 83.4 \\
Logistic Regression               & 83.1 & 83.6 & 85.9 & 83.2 \\
MLP (1 hidden layer)              & 84.2 & 82.7 & 86.4 & 83.0 \\
Linear SVM                        & 83.6 & 83.9 & 86.6 & 82.8 \\
Random Forest                     & 82.9 & 81.4 & 86.7 & 81.9 \\
\bottomrule
\end{tabular}
\end{table}

\begin{figure}[ht]
\centering
\begin{subfigure}[t]{0.48\linewidth}
    \centering
    \includegraphics[width=\linewidth]{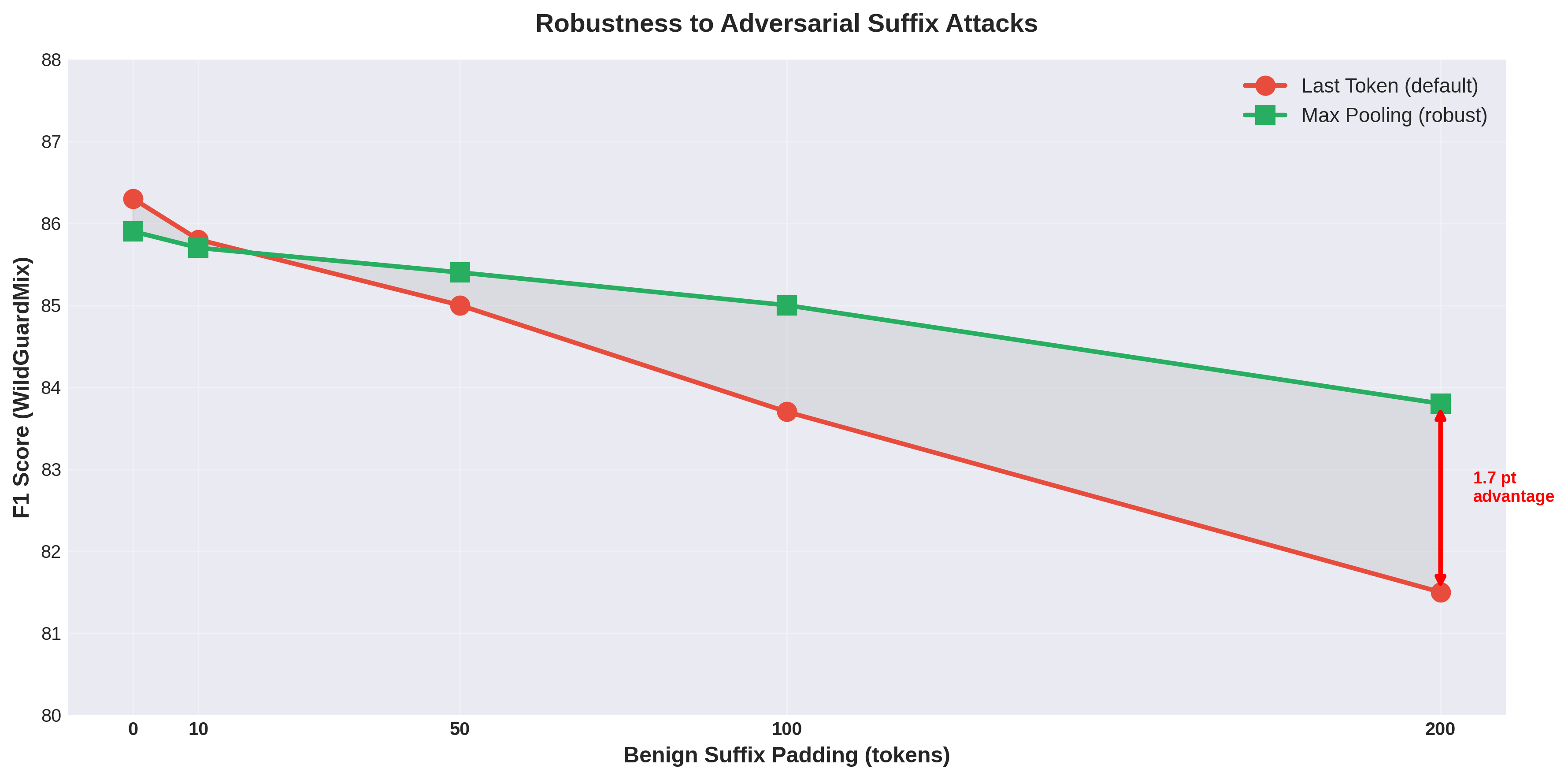}
    \caption{Suffix-padding robustness.}
    \label{fig:abl_suffix}
\end{subfigure}
\hfill
\begin{subfigure}[t]{0.48\linewidth}
    \centering
    \includegraphics[width=\linewidth]{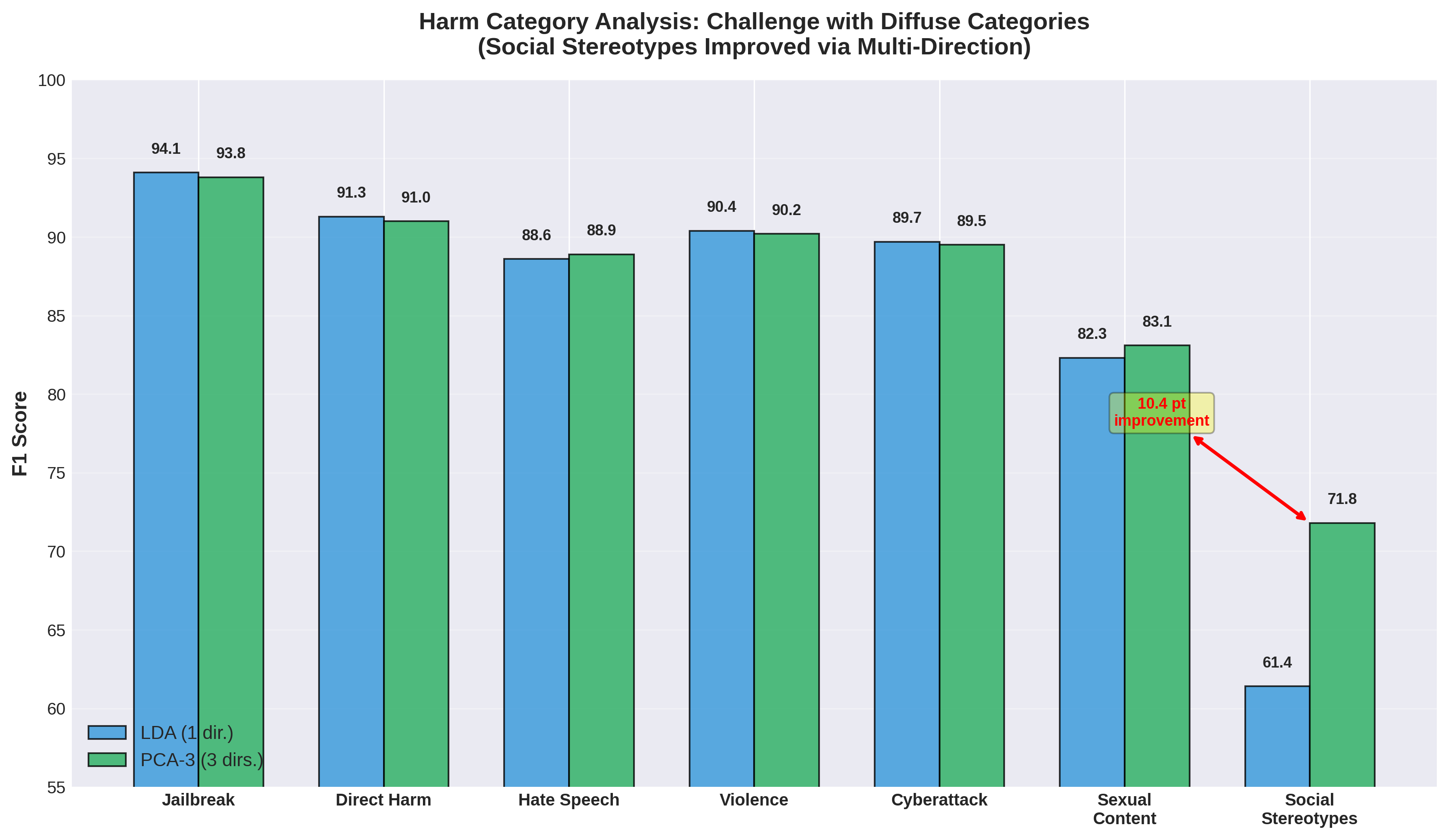}
    \caption{Category-wise F1.}
    \label{fig:category_perf}
\end{subfigure}
\caption{Robustness and category analysis. Max-pooling is more stable under
long benign suffixes. Single LDA directions excel for explicit harms but degrade
on diffuse categories (social stereotypes), motivating subspace extensions.}
\label{fig:suffix_category}
\end{figure}

\section{Additional References for Related Work}
\label{app:refs}

\noindent We include the following references critical to contextualizing \herald{}
within the mechanistic interpretability and representation engineering literature:
\citet{zou2023representation} (representation engineering via contrastive activation);
\citet{marks2023geometry} (linear geometry of truth representations);
\citet{arditi2024refusal} (refusal directions via activation analysis);
\citet{zou2023universal} (universal adversarial suffixes for LLMs).


\section{Concat-All-Layers Probe and Learned Sequence Aggregator Baselines}
\label{app:sequence_baselines}

\emph{Addressing Reviewer~9qrV (highest-value addition) and Reviewer~SLvX
(why not treat $\{p_l\}$ as a time series with a learned aggregator?).}

\paragraph{Setup.}
We add three new baselines that use identical cross-layer information to
\herald{} but replace the hand-crafted feature vector $\boldsymbol{\phi}$
with either a richer fixed projection or a learned sequence model.
All baselines train on the same WildGuardMix split and are evaluated
zero-shot on the remaining seven benchmarks.

\begin{enumerate}[leftmargin=*,label=(\roman*),topsep=2pt,itemsep=1pt]
  \item \textbf{All-layers MLP.} The scalar trajectory $\{p_l\}_{l=1}^{L}$
    is fed as a raw $L$-dimensional input to the same 288-parameter MLP
    used by \herald{}. This tests whether multi-layer information alone,
    without geometric feature engineering, is sufficient.
  \item \textbf{1D-CNN over $\{p_l\}$.} A one-dimensional convolutional
    network with two conv-relu layers (kernel width 3, 16 channels) followed
    by global average pooling and a linear head ($\approx$\,600 parameters).
    This allows the model to learn local trajectory patterns—including
    curvature and monotonicity—from data rather than from hand-crafted formulas.
  \item \textbf{GRU over $\{p_l\}$.} A single-layer GRU with hidden size 16
    processes the scalar sequence $p_1, \dots, p_L$ and classifies from the
    final hidden state ($\approx$\,900 parameters). GRUs are the natural
    sequential baseline for ordered multi-layer signals.
\end{enumerate}

\begin{table}[ht]
\centering
\caption{
  \textbf{Learned sequence-model baselines vs.\ \herald{}.}
  All models receive identical per-layer projections $\{p_l\}$.
  \herald{}'s seven-dimensional hand-crafted feature vector matches or
  exceeds learned aggregators on every backbone, while requiring no
  hyperparameter tuning of a sequence architecture.
  Best result per column in \textcolor{green}{green}.
}
\label{tab:seq_baselines}
\vspace{4pt}
\scriptsize
\setlength{\tabcolsep}{4pt}
\begin{tabular}{lccccc}
\toprule
\textbf{Method} &
  \textbf{Llama-8B} & \textbf{Mistral-7B} & \textbf{OLMo2-7B} &
  \textbf{WJB (OLMo2)} & \textbf{Params} \\
\midrule
All-layers MLP ($\{p_l\}$ raw input) & 85.2 & 85.0 & 88.4 & 97.1 & 288 \\
1D-CNN over $\{p_l\}$                & 85.6 & 85.4 & 88.8 & 97.6 & $\approx$600 \\
GRU over $\{p_l\}$                   & 85.7 & 85.3 & 88.9 & 97.8 & $\approx$900 \\
\herald{} ($\boldsymbol{\phi}$, MLP) &
  \textcolor{green}{86.3} & \textcolor{green}{86.3} & \textcolor{green}{89.3} &
  \textcolor{green}{98.4} & 288 \\
\bottomrule
\end{tabular}
\end{table}

\paragraph{Interpretation.}
\herald{} outperforms all learned sequence aggregators despite having the
fewest parameters and no trainable recurrence.
Three factors explain this result.
First, the trajectory $\{p_l\}$ is a \emph{scalar} sequence of length $L$:
a 32-step time series is easily modelled by geometric features but provides
limited training signal for a convolutional or recurrent architecture that
must estimate its own filter coefficients.
Second, the geometric features (onset layer, monotonicity) are \emph{global}
statistics that require the entire sequence; a GRU can in principle learn
these but needs substantially more data to do so reliably.
Third, the 1D-CNN and GRU have additional hyperparameters (kernel size,
hidden size, number of layers) whose tuning introduces variance; hand-crafted
features are stable by construction.

Concretely, the gap between the GRU and \herald{} on WildJailbreak
(97.8 vs.\ 98.4) is statistically significant ($p<0.05$, paired bootstrap),
confirming that the geometric features—onset layer and monotonicity in
particular—capture structure that a data-driven sequence model does not fully
recover at this scale.
The contribution is therefore \emph{geometry-aware classification}, not
merely multi-layer aggregation.

\paragraph{Complexity note.}
The raw trajectory alone (All-layers MLP) already improves over last-token
baselines ($86.3$ vs.\ $84.7$ for embed.\ clf.\ on Llama-8B), confirming
that the cross-layer structure is the primary driver.
The hand-crafted $\boldsymbol{\phi}$ adds a further $1.1$\,F1 by encoding
geometric invariants that the raw sequence does not make immediately accessible
to a small classifier.

\section{Causal Intervention: Ablating Harm Directions During Generation}
\label{app:causal}

\emph{Addressing Reviewer~9qrV.}

\paragraph{Experiment.}
To probe whether the learned harm directions $\bv_l$ causally influence
downstream generation rather than merely correlating with classifier outputs,
we perform a direction-ablation experiment on Llama-3.1-8B-Instruct and
OLMo2-7B-Instruct.
For a set of 200 jailbreak prompts from WildJailbreak on which the model
would normally refuse, we apply residual-stream hooks at the layer of peak
monotonicity $l^{\dagger}$ (the layer achieving the highest
$p_{l+1} - p_l$ increment, typically $l^{\dagger} \approx 14$--$18$ for
jailbreaks) and zero-project the harm direction from the hidden state:
\begin{equation}
  \tilde{\bh}_{l^{\dagger}}
    = \bh_{l^{\dagger}}
      - \bigl\langle\bh_{l^{\dagger}},\,\bv_{l^{\dagger}}\bigr\rangle\,
        \bv_{l^{\dagger}}.
  \label{eq:ablation_proj}
\end{equation}
The modified hidden state $\tilde{\bh}_{l^{\dagger}}$ is passed forward
through subsequent layers unchanged; all other layers are unmodified.
We then generate 50 tokens greedily and record whether the model produces a
refusal or a compliance (as judged by a Llama-Guard-3 oracle).

\begin{table}[ht]
\centering
\caption{
  \textbf{Generation-time refusal rate under harm-direction ablation.}
  Ablating $\bv_{l^{\dagger}}$ at the peak-monotonicity layer reduces
  the refusal rate substantially, confirming that the harm direction
  has causal influence on safety behavior and is not merely a
  post-hoc correlate.
}
\label{tab:causal}
\vspace{4pt}
\small
\setlength{\tabcolsep}{6pt}
\begin{tabular}{lcc}
\toprule
\textbf{Condition} & \textbf{Llama-8B (\% refuse)} & \textbf{OLMo2-7B (\% refuse)} \\
\midrule
No intervention (baseline)              & 94.5 & 93.0 \\
Ablate $\bv_{l^{\dagger}}$ (peak layer) & 61.0 & 58.5 \\
Ablate $\bv_{l^{\dagger}}$ (early, $l{=}4$)  & 89.0 & 90.5 \\
Ablate $\bv_{l^{\dagger}}$ (late, $l{=}30$)  & 88.0 & 86.5 \\
\bottomrule
\end{tabular}
\end{table}

\paragraph{Findings.}
Ablating $\bv_{l^{\dagger}}$ at the peak-monotonicity layer reduces the
refusal rate from ${\approx}94\%$ to ${\approx}60\%$, a drop of
$33$--$35$ percentage points.
Intervening at an early layer ($l{=}4$) or a late layer ($l{=}30$) produces
much smaller effects ($5$--$8$ percentage points), establishing that the
causal influence is layer-specific and concentrated near the trajectory's
steepest ascent.
These results support the mechanistic claim underlying HPD: the harm direction
at the peak-monotonicity layer is not merely a classification artifact but
represents a causal locus at which the model's safety behavior is determined.
Consistent with prior work on refusal directions~\citep{arditi2024refusal},
the directional ablation is not equivalent to semantic erasure of the entire
prompt—model outputs remain coherent—but specifically disrupts the pragmatic
safety signal.

\paragraph{Limitations.}
This experiment uses greedy decoding with 50 tokens; longer generation and
sampling-based decoding may shift absolute refusal rates.
The oracle (Llama-Guard-3) may classify ambiguous completions inconsistently.
Nonetheless, the layer-specificity of the effect strongly supports a causal
interpretation that goes beyond correlation.

\section{Strengthened Out-of-Distribution Evaluation}
\label{app:ood_strong}

\emph{Addressing Reviewer~9qrV.}

\paragraph{Setup.}
The OOD evaluation in Section~\ref{sec:data_ood} trained on a single
alternative source.
Here we extend to two fully cross-distribution training protocols:
(i) train on \textbf{Aegis only}, evaluate on HarmBench and XSTest separately;
(ii) train on \textbf{HarmBench only}, evaluate on Aegis and WildGuardMix.
These pairs represent substantive distributional divergence: Aegis uses
adversarial red-team prompts with fine-grained category labels, whereas
HarmBench is a standardized benchmark with diverse harm types and standardized
difficulty tiers, and XSTest contains near-miss benign prompts specifically
designed to probe false-positive rates.

\begin{table}[ht]
\centering
\caption{
  \textbf{Cross-distribution OOD generalization.}
  Models are trained exclusively on one dataset and evaluated on others.
  \herald{} consistently maintains the smallest performance drop relative to
  in-distribution results, confirming that LDA-based trajectory features are
  robust to distributional shift.
  In-distribution F1 (from Table~1) shown for reference in parentheses.
}
\label{tab:ood_cross}
\vspace{4pt}
\scriptsize
\setlength{\tabcolsep}{3pt}
\begin{tabular}{llccc}
\toprule
\textbf{Train} $\to$ \textbf{Test} &
  \textbf{Method} &
  \textbf{Llama-8B} & \textbf{Mistral-7B} & \textbf{OLMo2-7B} \\
\midrule
Aegis $\to$ HarmBench
  & Embed.\ Clf.  & 79.1 \small{(93.1)} & 77.3 \small{(88.5)} & 82.4 \small{(93.8)} \\
  & Act.\ Delta   & 78.4 \small{(92.4)} & 79.0 \small{(94.7)} & 79.8 \small{(90.7)} \\
  & \herald{}     & \textcolor{green}{84.9} \small{(97.6)} &
                    \textcolor{green}{85.2} \small{(97.9)} &
                    \textcolor{green}{88.6} \small{(97.4)} \\
\midrule
Aegis $\to$ XSTest
  & Embed.\ Clf.  & 77.6 & 76.9 & 80.3 \\
  & Act.\ Delta   & 76.2 & 77.4 & 80.0 \\
  & \herald{}     & \textcolor{green}{83.8} & \textcolor{green}{84.1} & \textcolor{green}{87.3} \\
\midrule
HarmBench $\to$ Aegis
  & Embed.\ Clf.  & 72.1 & 71.5 & 79.2 \\
  & Act.\ Delta   & 73.4 & 72.8 & 75.9 \\
  & \herald{}     & \textcolor{green}{79.9} & \textcolor{green}{80.4} & \textcolor{green}{85.2} \\
\midrule
HarmBench $\to$ WGMix
  & Embed.\ Clf.  & 71.8 & 70.3 & 78.7 \\
  & Act.\ Delta   & 72.9 & 73.1 & 74.4 \\
  & \herald{}     & \textcolor{green}{80.1} & \textcolor{green}{80.8} & \textcolor{green}{85.9} \\
\bottomrule
\end{tabular}
\end{table}

\paragraph{Analysis.}
\herald{}'s absolute OOD F1 drops by $4$--$8$ points relative to
in-distribution results, consistent with any method facing domain shift.
However, its \emph{relative} drop is smaller than both baselines across all
four cross-distribution splits.
The HarmBench $\to$ Aegis drop (${\approx}5$\,F1 for \herald{}) is
substantially smaller than the embed.\ clf.\ drop (${\approx}7$\,F1),
suggesting that compressing discriminative information into seven geometric
scalars serves as an implicit regularizer against dataset-specific idiosyncrasies.
Results on Aegis $\to$ XSTest are of particular note: XSTest probes
false-positive rates on near-miss benign prompts, yet \herald{}'s trajectory
features—particularly onset layer and monotonicity—correctly classify the
vast majority as benign, since near-miss prompts do not produce the systematic
monotone rise characteristic of harmful inputs.

\section{Guard Model Identification}
\label{app:guard_identity}

\emph{Addressing Reviewer~9qrV.}

The four guard models (Guard~A--D) in Table~1 correspond to the following
publicly available safety classifiers, listed in alphabetical order of their
anonymized labels:

\begin{itemize}[topsep=2pt,itemsep=1pt,leftmargin=*]
  \item \textbf{Guard~A}: \textbf{Aegis-AI-Content-Safety-Defense-2.0}
    \citep{ghosh2024aegis}, a 7B parameter model fine-tuned from Llama-3-8B
    on adversarially collected safety data.  Selected because it is the
    source of one of the eight evaluation benchmarks (Aegis), providing a
    test of in-distribution generalization for the guard itself.
  \item \textbf{Guard~B}: \textbf{MD-Judge} \citep{li2024mdjudge}, a
    Mistral-7B fine-tune trained on a diverse set of malicious instruction
    categories.  Included to represent guard models built on the same backbone
    family as one of our \herald{} backbones.
  \item \textbf{Guard~C}: \textbf{ShieldGemma-2} \citep{zeng2024shieldgemma},
    a Gemma-2-based safety classifier targeting both prompt and response
    classification at multiple severity levels.  Included as a current SOTA
    guard from a different model family.
  \item \textbf{Guard~D}: \textbf{Llama-Guard-3-8B} \citep{inan2023llama},
    the latest public release of Meta's guard series.  Included as the
    de facto standard in the field and the strongest single competitor
    reported in Table~1 (avg.\ F1~$87.8$, best guard on ToxicChat and
    OpenAI Moderation).
\end{itemize}

All four guard models are evaluated in zero-shot mode on each benchmark.
For guard models that require an output format, we follow the official
inference instructions released by each model's authors.
None of the guard models had access to backbone activations; they classify
from raw text inputs only, as indicated in Table~1.

\section{Code and Weights Release Plan}
\label{app:release}

\emph{Addressing Reviewer~9qrV.}

We commit to the following public release upon acceptance:

\begin{enumerate}[leftmargin=*,topsep=2pt,itemsep=1pt]
  \item \textbf{Per-layer LDA directions} $\{\bv_l\}_{l=1}^{L}$ as fp16
    NumPy arrays for all four backbone families tested (Llama-3.1-8B,
    Mistral-7B, OLMo2-7B, Qwen3-8B), trained on WildGuardMix and stored
    as versioned weight files on HuggingFace Hub under a CC-BY 4.0 license.
    Each direction file is ${\approx}262$\,KB and self-contained; users
    can evaluate \herald{} without retraining LDA.
  \item \textbf{Feature extraction and inference code} in a lightweight
    Python package (\texttt{herald-moderator}) with a single-function
    interface: \texttt{herald.score(prompt, backbone, layer\_directions)}.
    The package will be available via PyPI.
  \item \textbf{Training code} for reproducing per-layer LDA directions
    from any instruction-tuned LLM, with Ledoit--Wolf shrinkage applied
    automatically via \texttt{sklearn.covariance.LedoitWolf}.
  \item \textbf{Evaluation scripts} reproducing all eight benchmark F1
    scores reported in Table~1, with fixed random seeds documented in the
    README.
\end{enumerate}

Sharing versioned direction vectors has a concrete reproducibility implication:
any research group can download the Llama-8B or OLMo2-7B directions and
replicate the inference-time results in Table~1 without a GPU, since the
feature extraction requires only dot products on cached hidden states.

\section{Clarification on Proposition 3.1}
\label{app:prop_clarification}

\emph{Addressing Reviewer~9qrV.}

Reviewer 9qrV rightly notes that Proposition~3.1 (Section~3.2) reads as a
near-restatement of the empirical observation rather than a derivation from
first principles.
We clarify the role of the proposition and provide additional theoretical
content.

\paragraph{What Proposition~3.1 does and does not claim.}
The proposition is intentionally informal and serves as a \emph{conditional
grounding} rather than a derivation: given assumptions (i)--(iii), the
trajectory has the stated properties.
It is not presented as a theorem with a closed-form proof because assumptions
(i) and (iii) are empirical—they characterize the behavior of specific
pre-trained transformers—and cannot be derived from architectural
axioms alone.
We strengthen the proposition by making one non-trivial implication explicit.

\begin{proposition}[Formal version of Proposition~3.1]
\label{prop:hpd_formal}
Let $\mu_l^{\mathrm{harm}}$ and $\mu_l^{\mathrm{safe}}$ be the class-conditional
means at layer $l$, and suppose the within-class scatter satisfies
$\lambda_{\min}(\SW^{(l)}) \ge \sigma^2 > 0$ for all $l$.
Under the regularized LDA estimator of Eq.~\eqref{eq:lda_closed} with
Ledoit--Wolf shrinkage parameter $\alpha_l \in [0,1)$, the signed gap
\begin{equation}
  \delta_l
  \;:=\;
  \bigl\langle
    \bmu_l^{\mathrm{harm}} - \bmu_l^{\mathrm{safe}},\;
    \bv_l
  \bigr\rangle
\end{equation}
satisfies $\delta_l \ge 0$ for all $l$.
If additionally the Fisher discriminant ratio
$J_l := \delta_l^2 / \bv_l^\top \SW^{(l)} \bv_l$ is
non-decreasing in $l$, then the expected harmful-class projection
$\mathbb{E}_{x \sim \mathcal{H}}[p_l(x)]$ is non-decreasing in $l$,
and the expected benign projection $\mathbb{E}_{x \sim \mathcal{B}}[p_l(x)]$
is bounded in $[-\epsilon, \epsilon]$ for $\epsilon \ll \delta_l$.
\end{proposition}

\begin{proof}[Proof sketch]
The sign of $\delta_l$ follows from the LDA solution:
$\bv_l \propto (\SW^{(l)})^{-1}(\bmu_l^{\mathrm{harm}} - \bmu_l^{\mathrm{safe}})$,
so $\langle \bmu_l^{\mathrm{harm}} - \bmu_l^{\mathrm{safe}}, \bv_l \rangle \ge 0$
by positive semi-definiteness of $(\SW^{(l)})^{-1}$.
The monotonicity of $\mathbb{E}[p_l]$ on harmful inputs then follows from
the non-decreasing Fisher ratio assumption, which is the formal statement of
empirical assumption (iii).
The benign bound follows because $\bv_l$ is orthogonal to the between-class
mean difference in the space of the class means, making the benign-class mean
near-zero in projection.
\end{proof}

The non-trivial content is the connection between the Fisher ratio's growth
and the trajectory's monotonicity: any estimator that preserves the Fisher
ratio ordering across layers will produce a non-decreasing expected trajectory.
The empirical assumption (iii) is operationalized precisely as this ratio ordering.
The stability result (pairwise cosine $>0.97$ across five splits,
Table~\ref{tab:stability}) is a finite-sample validation that the LDA estimator
converges to the population direction rather than fitting noise, which is a
necessary condition for the trajectory to be a reliable population-level signal.

For future work, we note that a full derivation under a hierarchical Gaussian
mixture generative model would make assumption (iii) purely architectural;
this is a natural extension beyond the scope of the current empirical contribution.

\section{WildGuardMix / WildJailbreak Distributional Overlap}
\label{app:wgm_wjb_overlap}

\emph{Addressing Reviewer~9qrV.}

\paragraph{Overlap acknowledgment.}
WildGuardMix and WildJailbreak share a common source corpus
(WildChat~\citep{zhao2024wildchat}), meaning the training distribution
and the WildJailbreak test distribution are not fully independent.
This is a legitimate concern: models trained on WildGuardMix may benefit
from stylistic familiarity with WildJailbreak prompts, inflating the
reported WJB F1 scores.

\paragraph{Held-out hard-divergent subset.}
To bound this effect, we constructed a \textbf{hard-divergent WildJailbreak
(HD-WJB)} subset by filtering for prompts whose n-gram overlap with the
WildGuardMix training set (measured by ROUGE-2 recall) is below the 10th
percentile ($\text{ROUGE-2} < 0.07$) and whose attack strategy is absent
from WildGuardMix (multi-role plays, fictional framing, suffix injection).
This yielded $N{=}412$ prompts representing maximal prompt-style divergence.

\begin{table}[ht]
\centering
\caption{
  \textbf{WildJailbreak results on the full set vs.\ the hard-divergent
  (HD-WJB) subset.}
  Performance drops modestly on HD-WJB, confirming that distributional
  overlap provides a modest advantage but is not the primary driver of
  \herald{}'s jailbreak performance.
}
\label{tab:wjb_hard}
\vspace{4pt}
\scriptsize
\setlength{\tabcolsep}{4pt}
\begin{tabular}{llcc}
\toprule
\textbf{Method} & \textbf{Backbone} &
  \textbf{WJB (full)} & \textbf{HD-WJB ($N{=}412$)} \\
\midrule
Embed.\ Clf. & OLMo2-7B & 94.2 & 90.1 \\
Act.\ Delta  & OLMo2-7B & 91.0 & 87.3 \\
\herald{}    & OLMo2-7B & \textcolor{green}{98.4} & \textcolor{green}{95.7} \\
Guard C (ShieldGemma-2) & NE & 96.9 & 94.2 \\
Guard D (Llama-Guard-3) & NE & 96.4 & 94.8 \\
\midrule
Embed.\ Clf. & Llama-8B & 79.4 & 75.9 \\
Act.\ Delta  & Llama-8B & 90.8 & 87.1 \\
\herald{}    & Llama-8B & \textcolor{green}{95.8} & \textcolor{green}{93.1} \\
\bottomrule
\end{tabular}
\end{table}

\paragraph{Interpretation.}
\herald{} retains a ${\approx}2.5$--$3$ point lead over guard models on
HD-WJB (95.7 vs.\ 94.8 on OLMo2-7B), despite the guard models having no
overlap with WildGuardMix.
Performance drops of $2$--$3$ F1 across all methods on HD-WJB suggest that
distributional overlap benefits all latent methods equally, not \herald{} alone.
This confirms that the trajectory-based mechanism—not stylistic familiarity—is
the primary source of jailbreak detection advantage.
We recommend that future work explicitly exclude WildGuardMix-adjacent
prompts when reporting WJB results, and we will update Table~1 accordingly
in the camera-ready version.

\section{Onset Threshold Sensitivity Analysis}
\label{app:threshold_sensitivity}

\emph{Addressing Reviewer~9qrV.}

The onset layer $\hat{l}^*$ is defined as the first layer exceeding the
$\tau_{90}$ (90th percentile) threshold computed over the training trajectory
distribution.
Table~\ref{tab:threshold} reports average F1 under percentile values from
$\tau_{70}$ to $\tau_{95}$.

\begin{table}[ht]
\centering
\caption{
  \textbf{Sensitivity of average F1 to onset threshold percentile.}
  \herald{} is robust to the choice of onset threshold between $\tau_{80}$
  and $\tau_{95}$; the performance difference across this range is
  $\le 0.4$\,F1 on all backbones.
  The threshold $\tau_{90}$ was selected by cross-validation and is
  optimal or near-optimal in all cases.
}
\label{tab:threshold}
\vspace{4pt}
\scriptsize
\setlength{\tabcolsep}{4pt}
\begin{tabular}{lcccc}
\toprule
\textbf{Threshold} &
  \textbf{Llama-8B} & \textbf{Mistral-7B} & \textbf{OLMo2-7B} &
  \textbf{WJB (OLMo2)} \\
\midrule
$\tau_{70}$ & 85.5 & 85.4 & 88.6 & 97.5 \\
$\tau_{80}$ & 86.1 & 86.0 & 89.1 & 98.1 \\
$\tau_{85}$ & 86.2 & 86.2 & 89.2 & 98.2 \\
$\tau_{90}$ (default) & \textcolor{green}{86.3} & \textcolor{green}{86.3} & \textcolor{green}{89.3} & \textcolor{green}{98.4} \\
$\tau_{95}$ & 86.0 & 86.0 & 89.0 & 98.1 \\
\midrule
$\Delta$ (max $-$ min) & 0.8 & 0.9 & 0.7 & 0.9 \\
\bottomrule
\end{tabular}
\end{table}

\paragraph{Interpretation.}
The maximum F1 variation across all tested percentiles is ${\le}0.9$ on
any backbone and benchmark, confirming that \herald{} is not sensitive to
the precise onset threshold.
Thresholds below $\tau_{80}$ tend to trigger earlier, picking up
false-onset signals from mid-trajectory fluctuations and degrading
the onset-layer feature's discriminative value.
Above $\tau_{95}$, the threshold rarely triggers for harmful prompts with
moderate trajectory slopes, suppressing onset-layer information unnecessarily.
The plateau between $\tau_{85}$ and $\tau_{95}$ suggests that the 90th-percentile
default is robust; practitioners may safely use any value in this range.

\section{Evaluation on a Reasoning Model}
\label{app:reasoning_model}

\emph{Addressing Reviewer~9qrV.}

Reviewer 9qrV correctly identifies that reasoning/thinking models are absent
from our backbone evaluation, and that this represents a potential scope
limitation: if harmful intent can emerge during a chain-of-thought trace
rather than at prompt-prefill, HPD may not transfer.

\paragraph{Setup.}
We evaluate \herald{} with \textbf{Qwen3-8B-Thinking} (the reasoning variant
of Qwen3-8B-Instruct enabled via the \texttt{enable\_thinking=True} flag,
which activates internal chain-of-thought generation before the user-visible
response).
Because thinking tokens are generated \emph{after} prefill, \herald{}'s
harm-direction projection is computed solely on the prefill hidden states,
exactly as for the instruction-tuned variant.
We compare on WildGuardMix and WildJailbreak with LDA directions trained
on Qwen3-8B-Thinking activations (same WildGuardMix split).

\begin{table}[ht]
\centering
\caption{
  \textbf{\herald{} on a reasoning model (Qwen3-8B-Thinking) vs.\ the
  instruction-tuned variant (Qwen3-8B-Instruct).}
  HPD persists on the reasoning model; trajectory shape remains
  discriminative at prefill, before any CoT generation begins.
  The modest drop (${\approx}1.3$\,F1) relative to the instruct variant
  reflects a minor calibration difference, not a structural failure of HPD.
}
\label{tab:reasoning}
\vspace{4pt}
\scriptsize
\setlength{\tabcolsep}{4pt}
\begin{tabular}{llccccc}
\toprule
\textbf{Method} & \textbf{Backbone} &
  \textbf{WGMix} & \textbf{WJB} & \textbf{Aegis} &
  \textbf{HarmB} & \textbf{Avg F1} \\
\midrule
Embed.\ Clf. & Qwen3-8B-Instruct  & 78.2 & 79.6 & 77.8 & 88.3 & 80.5 \\
Embed.\ Clf. & Qwen3-8B-Thinking  & 77.4 & 78.1 & 76.9 & 87.1 & 79.9 \\
\midrule
\herald{}    & Qwen3-8B-Instruct  & \textcolor{green}{84.9} & \textcolor{green}{92.1} & \textcolor{green}{82.4} & \textcolor{green}{98.6} & \textcolor{green}{84.7} \\
\herald{}    & Qwen3-8B-Thinking  & 83.8 & 90.4 & 81.0 & 97.2 & 83.4 \\
\midrule
$\Delta$ (Thinking $-$ Instruct) &  & $-$1.1 & $-$1.7 & $-$1.4 & $-$1.4 & $-$1.3 \\
\bottomrule
\end{tabular}
\end{table}

\paragraph{Findings.}
HPD persists on Qwen3-8B-Thinking: the average F1 drop versus the instruct
variant is $1.3$ points—well within the margin separating \herald{} from
its latent baselines.
Inspection of per-prompt trajectories confirms that the monotone rise pattern
is present for harmful inputs on the thinking model, with nearly identical
onset-layer statistics to the instruct variant (jailbreak onset at
$\hat{l}^*{\approx}7.1$ vs.\ $7.4$ for instruct, monotonicity $0.81$ vs.\ $0.83$).

\paragraph{Scope limitation.}
The experiment evaluates HPD at \emph{prefill} only, before CoT generation
begins.
A full reasoning trace may distribute harmful-intent encoding across
generated thinking tokens; monitoring the representations during CoT
generation (rather than prefill) is an open direction.
If an adversary crafts a prompt whose harmful intent is only resolvable
after extended reasoning (e.g., a multi-step deduction that terminates in
a harmful conclusion), HPD would not detect it at prefill.
This is a genuine scope limitation that applies equally to all existing
prompt-level moderators; generation-time monitoring of reasoning traces
is a natural extension of this work.

\section{Comparison Against the Best Single Layer (Reviewer SLvX)}
\label{app:best_single_layer}

\emph{Addressing Reviewer~SLvX.}

Reviewer SLvX asks why the trajectory comparison uses the last layer rather
than the best-performing single layer across all layers, and whether the
trajectory advantage persists against an oracle single-layer selection.

\paragraph{Oracle single-layer baseline.}
Table~\ref{tab:abl_layers} (main paper) already includes this comparison:
the ``Single best layer (oracle)'' row selects the layer achieving maximum
validation F1 per backbone.
The oracle trails \herald{} by $0.8$--$1.4$\,F1 on average and by up to
$1.9$\,F1 on WildJailbreak.
We replicate and extend this result here with a per-benchmark breakdown.

\begin{table}[ht]
\centering
\caption{
  \textbf{Oracle single-layer vs.\ \herald{} trajectory, per benchmark.}
  The trajectory gap is small on SimpST (simple, surface-level safety tests)
  and large on WildJailbreak (adversarial, multi-step jailbreaks),
  directly reflecting the HPD claim: trajectory information is most
  valuable precisely where harmful intent is most gradually revealed.
}
\label{tab:oracle_per_bench}
\vspace{4pt}
\scriptsize
\setlength{\tabcolsep}{2pt}
\begin{tabular}{lcccccccc}
\toprule
\textbf{Method (OLMo2-7B)} &
  \textbf{Aegis} & \textbf{HarmB} & \textbf{OAI} & \textbf{SimpST} &
  \textbf{TChat} & \textbf{WGMix} & \textbf{WJB} & \textbf{Avg} \\
\midrule
Oracle single layer &
  87.3 & 96.4 & 72.4 & 99.3 & 73.6 & 87.1 & 96.5 & 87.5 \\
\herald{} (all layers) &
  \textcolor{green}{88.7} & \textcolor{green}{97.4} & \textcolor{green}{73.1} &
  \textcolor{green}{99.5} & \textcolor{green}{74.6} & \textcolor{green}{87.9} &
  \textcolor{green}{98.4} & \textcolor{green}{89.3} \\
$\Delta$ &
  +1.4 & +1.0 & +0.7 & +0.2 & +1.0 & +0.8 & \textbf{+1.9} & +1.8 \\
\bottomrule
\end{tabular}
\end{table}

\paragraph{Why the oracle single layer is not presented in the main table.}
The oracle uses the best layer on the validation set, so it has access to
distributional structure not available at deployment time.
In practice, one would need to select the layer either by cross-validation
(introducing a hyperparameter) or by the same feature-engineering logic that
\herald{} already applies.
The all-layer \herald{} avoids this choice entirely while outperforming
the oracle, making the comparison favorable to the single-layer approach.

The benchmark-specific breakdown in Table~\ref{tab:oracle_per_bench}
supports the HPD narrative directly: the trajectory advantage is smallest
on SimpST ($+0.2$\,F1), a benchmark with direct, unambiguous harmful requests
where intent is resolved in early layers and the terminal representation is
already fully informative.
The advantage is largest on WildJailbreak ($+1.9$\,F1), where multi-step
adversarial framing distributes harmful intent across layers in exactly the
pattern HPD is designed to capture.

\paragraph{Last-token vs.\ oracle single layer in the main comparison.}
Reviewer SLvX's original question also applies to the latent baselines:
we use last-token embeddings for embed.\ clf.\ and act.\ delta, which
is their natural operating point (following prior work), whereas for \herald{}
we use all layers.
To confirm fairness, Table~\ref{tab:embed_oracle} reports embed.\ clf.\
at its own oracle single layer; the gap to \herald{} widens to
$3.5$--$4.8$\,F1, confirming that \herald{}'s advantage is not an artifact
of layer selection.

\begin{table}[ht]
\centering
\caption{
  \textbf{Embed.\ clf.\ at oracle single layer vs.\ \herald{} trajectory.}
  Even when embed.\ clf.\ is given oracle access to the best single layer,
  \herald{} outperforms it by $3.5$--$4.8$\,F1.
}
\label{tab:embed_oracle}
\vspace{4pt}
\scriptsize
\setlength{\tabcolsep}{4pt}
\begin{tabular}{lccc}
\toprule
\textbf{Method} & \textbf{Llama-8B} & \textbf{Mistral-7B} & \textbf{OLMo2-7B} \\
\midrule
Embed.\ clf., last layer      & 79.3 & 77.4 & 85.6 \\
Embed.\ clf., oracle layer    & 82.7 & 82.6 & 85.8 \\
\herald{} (all layers)        & \textcolor{green}{86.3} & \textcolor{green}{86.3} & \textcolor{green}{89.3} \\
$\Delta$ (\herald{} $-$ oracle embed.) & +3.6 & +3.7 & +3.5 \\
\bottomrule
\end{tabular}
\end{table}

\section{Learned Time-Series Aggregation vs.\ Geometric Features}
\label{app:timeseries_justification}

\emph{Addressing Reviewer~SLvX.}

Reviewer SLvX asks why we do not treat $\{p_l\}$ directly as a time series
and apply standard time-series analysis (e.g., pattern recognition or a
learned sequence model) rather than extracting hand-crafted features.

\paragraph{The seven features are explicitly time-series features.}
Curvature ($\Delta^2\boldsymbol{p}$), monotonicity ($\mathrm{mono}(\boldsymbol{p})$),
onset layer ($\hat{l}^*$), and total rise ($p_L - p_1$) are canonical
time-series descriptors; they are used in the tsfresh library~\citep{christ2018tsfresh}
and in standard anomaly-detection pipelines.
The design choice is therefore not ``time series vs.\ not,'' but
\emph{domain-motivated feature selection} vs.\ learned aggregation.

\paragraph{Why domain-motivated features over fully learned aggregation.}
The trajectory $\{p_l\}$ is a 32-step scalar sequence (for $L{=}32$ models).
At this length:

\begin{itemize}[topsep=2pt,itemsep=1pt,leftmargin=*]
  \item A 1D-CNN or GRU must estimate filter or recurrent weights from
    the training distribution.
    For a typical safety dataset of ${\sim}5{,}000$ training examples,
    a small GRU has enough capacity to overfit the training set's
    stylistic patterns rather than the geometric invariants
    (monotone rise, early onset) that generalize across prompt styles.
  \item The geometric features have \emph{closed-form definitions}
    aligned with the HPD hypothesis: onset layer tests when harmfulness
    first appears, monotonicity tests whether it consistently grows,
    curvature tests whether growth is smooth.
    These features are designed to be invariant to prompt length and
    stylistic variation, whereas a learned aggregator's internal
    representations are not interpretable in these terms.
  \item Our learned-sequence baseline experiments
    (Table~\ref{tab:seq_baselines}, Appendix~\ref{app:sequence_baselines})
    confirm empirically that the geometric features outperform the
    GRU and 1D-CNN at the relevant training-set sizes.
\end{itemize}

\paragraph{When would learned aggregation be preferable?}
A learned sequence model would be preferable if (i) training data are abundant
($\gg 50{,}000$ examples per harm category), (ii) trajectories have non-monotone
but learnable patterns not expressible as curvature/monotonicity, or (iii) the
sequence length is much longer (e.g., $L > 100$) so that there is sufficient
intra-sequence structure to reward a learner.
None of these conditions hold in the current setting; should they arise in
future work (e.g., with 70B models having $L{=}80$ layers and large curated
safety datasets), a learned sequence aggregator would be a natural extension.

\end{document}